\documentclass[journal, 10pt]{IEEEtran}

\usepackage{cite}
\usepackage{amsmath,amssymb,amsfonts}
\usepackage{algorithmic}
\usepackage{graphicx}
\usepackage{textcomp}
\usepackage{xcolor}
\usepackage{subfig}
\usepackage{amsthm}
\usepackage{mathrsfs}

\usepackage{pdfpages}
\usepackage{colortbl}
\definecolor{mygray}{gray}{0.85} 
\usepackage{array}
\usepackage{caption}
\usepackage{setspace}
\usepackage{flushend}  
\usepackage[linesnumbered,boxed,ruled,commentsnumbered]{algorithm2e}

\usepackage[ruled,linesnumbered]{algorithm2e}

\usepackage{amsmath,bm}
\usepackage{cite}
\usepackage{amsmath,amssymb,amsfonts}
\usepackage{algorithmic}
\usepackage{amsthm}
\usepackage{graphicx}
\usepackage{textcomp}
\usepackage{xcolor}

\begin{document}
\begin{spacing}{1.0}
%
\title{
Multi-Agent Deep Reinforcement Learning for Safe Autonomous Driving with RICS-Assisted MEC
}
\author{ 
Xueyao~Zhang,~\IEEEmembership{Student Member,~IEEE,} Bo~Yang,~\IEEEmembership{Member,~IEEE,} Xuelin Cao,~\IEEEmembership{Member,~IEEE,}\\ Zhiwen Yu,~\IEEEmembership{Senior Member,~IEEE,} George C. Alexandropoulos,~\IEEEmembership{Senior Member,~IEEE,}   \\ Yan Zhang,~\IEEEmembership{Fellow,~IEEE},  M\'erouane Debbah,~\IEEEmembership{Fellow,~IEEE},  and Chau~Yuen,~\IEEEmembership{Fellow,~IEEE}
\thanks{X. Zhang and B. Yang are with the School of Computer Science, Northwestern Polytechnical University, Xi'an, Shaanxi, 710129, China (email: yang$\_$bo@nwpu.edu.cn, 2024263006@mail.nwpu.edu.cn). 

Z. Yu is with the School of Computer Science, Northwestern Polytechnical University, Xi'an, Shaanxi, 710129, China, and Harbin Engineering University, Harbin, Heilongjiang, 150001, China (email: zhiwenyu@nwpu.edu.cn).

X. Cao is with the School of Cyber Engineering, Xidian University, Xi'an, Shaanxi, 710071, China (email: caoxuelin@xidian.edu.cn). 

G. C. Alexandropoulos is with the Department of Informatics and Telecommunications, National and Kapodistrian University of Athens, 16122 Athens, Greece (email: alexandg@di.uoa.gr). 

Y. Zhang is with the Department of Informatics, University of Oslo, 0316 Oslo, Norway (email: anzhang@ieee.org).  

M. Debbah is with  KU 6G Research Center, Department of Computer and Information Engineering, Khalifa University, Abu Dhabi 127788, UAE (email: merouane.debbah@ku.ac.ae)

C. Yuen is with the School of Electrical and Electronics Engineering, Nanyang Technological University, Singapore (email: chau.yuen@ntu.edu.sg).}
}


\maketitle

\begin{abstract}
Environment sensing and fusion via onboard sensors are envisioned to be widely applied in future autonomous driving networks. This paper considers a vehicular system with multiple self-driving vehicles that is assisted by multi-access edge computing (MEC), where image data collected by the sensors is offloaded from cellular vehicles to the MEC server using vehicle-to-infrastructure (V2I) links. Sensory data can also be shared among surrounding vehicles via vehicle-to-vehicle (V2V) communication links. To improve spectrum utilization, the V2V links may reuse the same frequency spectrum with V2I links, which may cause severe interference. To tackle this issue, we leverage reconfigurable intelligent computational surfaces (RICSs) to jointly enable V2I reflective links and mitigate interference appearing at the V2V links. Considering the limitations of traditional algorithms in addressing this problem, such as the assumption for quasi-static channel state information, which restricts their ability to adapt to dynamic environmental changes and leads to poor performance under frequently varying channel conditions, in this paper, we formulate the problem at hand as a Markov game. Our novel formulation is applied to time-varying channels subject to multi-user interference and introduces a collaborative learning mechanism among users. The considered optimization problem is solved via a driving safety-enabled multi-agent deep reinforcement learning (DS-MADRL) approach that capitalizes on the RICS presence. Our extensive numerical investigations showcase that the proposed reinforcement learning approach achieves faster convergence and significant enhancements in both data rate and driving safety, as compared to various state-of-the-art benchmarks.

\end{abstract}


\IEEEpeerreviewmaketitle

\section{Introduction}
\IEEEPARstart{W}{ith} the rapid advancement of autonomous driving technology, the perception and decision-making capabilities of vehicles have become crucial for ensuring safe autonomous driving. However, autonomous vehicles (AVs) are required to process vast amounts of data collected from onboard sensors in real-time, imposing significant challenges to both computational resources and communication capabilities  \cite{Auto01}\cite{Auto02}. This challenge is exacerbated in high-density vehicular ad hoc networks (VANETs) \cite{communication_Auto01}\cite{communication_Auto02}, where latency-sensitive tasks —ranging from collision avoidance to real-time path planning— need to be executed within sub-second intervals. In this context, onboard devices are required to perform extensive sensing and computational tasks within a short time frame, while simultaneously relying on Vehicle-to-Infrastructure (V2I) and Vehicle-to-Vehicle (V2V) communication links to transmit substantial volumes of data, all amidst the competition for limited spectral resources. Existing research has predominantly focused on enhancing modeling algorithms to improve computational accuracy or addressing real-world issues, such as path planning \cite{Auto03}. However, limited attention has been given to the interference on the same spectrum in high-density vehicular environments and the fulfillment of dynamic computational requirements, which are critical for ensuring both efficiency and safety.

\begin{figure*}[t]
	\centering
	\includegraphics[width=1.88\columnwidth]{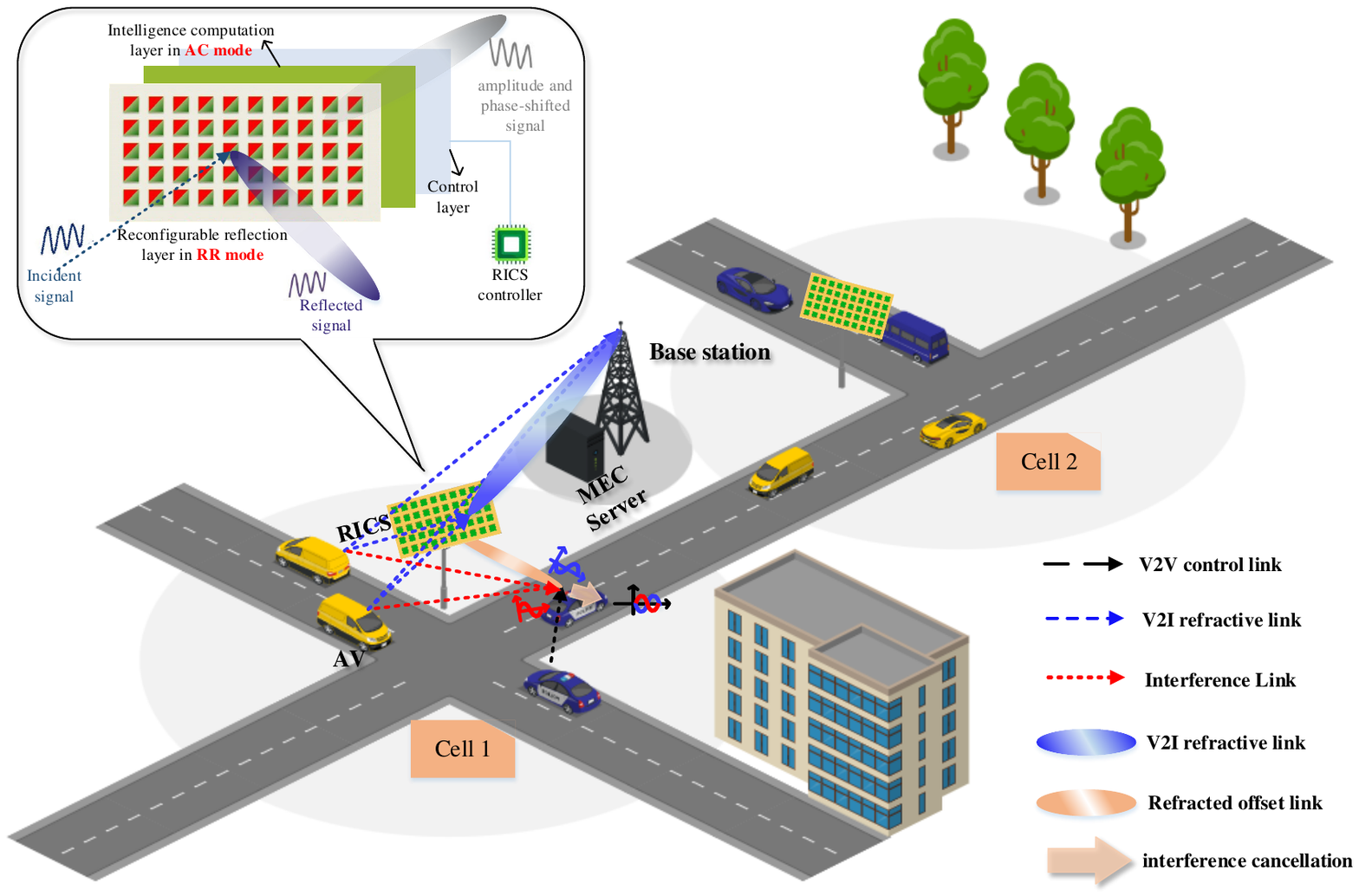}
	\caption{The considered autonomous driving system model incorporating an RICS.}
    \label{scene2}
\end{figure*}

By utilizing the dynamic modeling approach of Markov decision processes (MDPs), one can effectively capture the environmental conditions that change over time and make informed decisions in each time slot. This enables optimizing performance in rapidly evolving high-density vehicular network environments. To this end, in this paper, we adopt a DRL approach to effectively address the high coupling and complexity of the optimization variables. In particular, we present a novel autonomous driving network architecture based on the RICS technology and propose the driving safety-enabled multi-agent DRL (DS-MADRL) framework to address task offloading for AVs, spectrum sharing strategies, and joint optimization of the RICS parameters. It is demonstrated that our framework not only effectively enhances the spectrum efficiency and data rate of the considered system, but also significantly improves the real-time perception and safety of vehicles. Furthermore, compared to traditional algorithms, our proposed algorithm exhibits lower computational complexity and better adaptability. 

The contributions of this paper are summarized below:
\begin{itemize}
    \item A novel automated driving network system assisted by an RICS is presented. The system capitalizes on MEC to facilitate collaborative perception and decision-making between vehicles, aiming to address the limitations of computational resources and low communication efficiency faced by AVs. Our system model considers a multi-cell network architecture served by a single base station, where vehicles communicate with the BS via Vehicle-to-Infrastructure (V2I) links using Frequency Division Multiple Access (FDMA) for efficient channel allocation, thereby avoiding mutual interference. V2V communication enhances channel utilization by reusing AV channels, while local computational tasks can be partially offloaded to edge servers to alleviate local computational loads and reduce processing delays. RICS can simultaneously reflect and transmit signals to extend coverage while adjusting transmitted signals to mitigate interference introduced in V2V communication. We formulate a novel optimization problem incorporating parameters related to driving safety.
    \item We model our automated driving network optimization problem as a Markov game and solve it using a novel driving safety-enabled multi-agent deep reinforcement learning (DS-MADRL) framework. This framework innovatively combines hybrid action space optimization with safety-driven reward design. The proposed algorithm employs Q-decaying DQN (DDQN) to handle discrete actions and Multi-pass DQN (MP-DQN) networks, which is simpler to train and more stable in discrete-continuous hybrid spaces, to address continuous-discrete actions. Additionally, the design of the reward function integrates the safety of AVs and the reliability of V2V communication to ensure system safety in dynamic environments. Moreover, we consider time-varying channels and interference among multiple users and introduce a cooperative learning mechanism among users. By designing a joint reward function to enhance system security, we effectively address the complex policy coordination challenges faced by two types of agents within a heterogeneous action space. The centralized training with decentralized execution (CTDE) framework is adopted to facilitate centralized training and distributed execution~\cite{CTDE}, thus, reducing the communication and computational overhead during online execution.
    \item The convergence performance of the proposed algorithmic framework is verified through extensive simulation results, ensuring the safety of autonomous driving. We also investigate the robustness of the algorithm in both single- and multi-cell scenarios and explore the role of amplitude adjustment coefficients in interference cancellation. Our extensive comparisons with traditional optimization algorithms, as well as DQN and deep deterministic policy gradients (DDPG) algorithms, showcase the superiority of the proposed approach.
\end{itemize}

The paper is organized as follows: Section II introduces the RICS-assisted autonomous driving system model along with the channel model and our optimization problem formulation promoting driving safety. Section III presents the fundamentals of MADRL, models the optimization problem as a Markov game, and discusses the proposed DS-MADRL algorithm. Section VI verifies the proposed algorithm's effectiveness through simulation experiments and comparisons with state-of-the-art algorithms. The concluding remarks of the paper are included in Section V.

\begin{table}[t]
	\centering
	\caption{System model Parameters}
	\label{parameters1}
	\small
	\begin{tabular}{l   l }
		\noalign{\global\arrayrulewidth=0.3mm}
		\hline
		\textbf{Parameter}  &\textbf{Description } \\
		\hline
          $\mathcal{C}$   & Set of cells \\
		 $\mathcal{V}$  &Set of V2Vs \\
         $\mathcal{K}$    &Set of elements of the RICS  \\
         $\beta_k^{r}$    &Reflection amplitude of each element\\
         $\beta_k^{t}$   &Transmission amplitude of each element\\
         ${\mathbf{\Theta}}_r(l)$ & The reflection coefficient matrices \\
         ${\mathbf{\Theta}}_t(l)$ & The transmission coefficient matrices\\
         ${\textbf h}_{u,r} \in \mathbb{C}^{K \times 1}$ &the channels from AV to RICS\\
         ${\textbf h}_{r,b} \in \mathbb{C}^{K \times 1}$ &the channels from RICS to BS\\
         ${\textbf h}_{r,v} \in \mathbb{C}^{K \times 1}$ &the channels from RICS to V2V\\
		$\omega_{u,n}$  & channel sharing  \\
        $\rho_{u,c}$ &offloading ratio\\
        $s_{u,c}$     &The input data size for computation \\         
		$F_{u,c}$   & Resources allocated to each AV by BS\\
        $A_{u,c}$   &The AVs' inference accuracy\\
        $A_{b}$     &The BS's inference accuracy\\
		\noalign{\global\arrayrulewidth=0.3mm}
		\hline
	\end{tabular}
\end{table}	

\theoremstyle{Observation}
\newtheorem{observation}{\textit{Observation}}
\theoremstyle{Lemma} 
\newtheorem{lemma}{{\textit{Lemma}}}

\section{Related works}
\subsection{MEC in Vehicular Networks}
To address these challenges, Multi-Access Edge Computing (MEC) has been proposed as a solution to assist vehicles in offloading computation, while enhancing the efficiency of information transfer through V2I and V2V communications. This approach leverages the proximity of edge servers to reduce latency and improve reliability, making it particularly suitable for dynamic vehicular environments. In \cite{MEC01}, a distributed multi-hop task offloading decision model is proposed to optimize task execution efficiency in MEC and V2I by leveraging vehicles with idle computational resources. In addition, \cite{MEC02} investigates novel task offloading algorithms aimed at optimizing offloading and resource allocation strategies~\cite{offloading1, offloading2, offloading3}. However, these studies primarily focus on enhancing the offloading process, often assuming that interference from communication links is negligible, thus, failing to adequately account for the signal degradation caused by spectrum sharing~\cite{sharing1, sharing2, sharing3, sharing4}. This oversight complicates the ability of existing methods to effectively tackle signal degradation and fluctuations in communication quality in real-world dynamic traffic environments, which could potentially compromise the safety of autonomous driving.

\subsection{RIS for V2X Communications}
Reconfigurable intelligent surfaces (RIS) have demonstrated significant potential in enhancing communication quality, such as improving channel gain and reducing signal fading \cite{RIS00,RISa,RISb,RIS01,RIS02,RIS03}, while also offering advantages like low energy consumption and high energy efficiency \cite{RIS04,RISc}. Recent advancements have seen the emergence of innovative RIS structures, such as filtered reconfigurable intelligent computational
surface (RICS)~\cite{FRICS}, multi-layer RISs~\cite{multi}, hybrid simultaneous reflecting and sensing RISs~\cite{hybrid}, and simultaneous transmitting a reflecting (STAR) RISs~\cite{STAR}, which are beginning to explore additional functionalities and capabilities. This exploration signifies a growing interest in leveraging RIS technology within the field of vehicular networks, with recent studies focusing on resource allocation, communication reliability, channel estimation, as well as performance analysis in RIS-assisted V2X systems \cite{RIS-V2X01, RIS-V2X02}.
However, within the context of high-density vehicular networks, the unilateral reflective capability of conventional RISs limits severe signal coverage while being highly impacted by interference conditions. The innovative structure of an RICS in~\cite{RICS} has been shown to be configured in various ways to adapt to different scenarios. It is designed to simultaneously transmit and reflect signals, thereby improving coverage and adapting the V2I refraction link to mitigate interference on the V2V link. RICS optimizes the quality of wireless communications and adaptively modifies signals due to its computational ability. This capability enhances perception and decision-making efficiency for automated driving in complex dynamic environments.

\subsection{DRL Foundations}
The evolution of deep reinforcement learning has fundamentally transformed optimization in complex dynamic systems. Building upon the foundational deep Q-network (DQN) that integrated experience replay and target networks for stable value estimation, subsequent innovations addressed critical limitations in real-world deployment. Decaying DQN (DDQN) \cite{DDQN} employs $\varepsilon$-greed to adaptively balance exploration-exploitation tradeoffs, ensuring smooth policy convergence in dynamic networks with time-varying environments.
For continuous control, Deep Deterministic Policy Gradient (DDPG) \cite{DDPG} combined actor-critic architectures with off-policy learning to efficiently optimize policies. The evolution of hybrid action space optimization saw critical breakthroughs with Parameterized DQN (P-DQN) \cite{P-DQN}, which decoupled discrete and continuous policy networks for joint action learning, yet faced gradient interference between decision dimensions. Multi-Pass DQN (MP-DQN) \cite{MP-DQN} resolved this by isolating action-specific gradients through masked basis vector propagation, enabling stable training in complex decision spaces. In multi-agent scenarios, Centralized Training with Decentralized Execution (CTDE) \cite{CTDE1} emerged as a foundational paradigm, allowing collaborative policy learning while preserving decentralized execution efficiency—a crucial balance for systems requiring both coordination and operational autonomy.

\subsection{DRL for Communication Systems}
Despite the latter advancements, the integration of RICS into autonomous driving networks introduces a multi-dimensional optimization challenge. Typically, traditional optimization methods, such as Lagrange multipliers \cite{Lag}, KKT conditions \cite{KKT}, manifold optimization (MO) \cite{MO}, dyadic optimization, and alternating optimization \cite{AO}, are commonly employed to tackle the modeled non-convex problems. However, these problems are often NP-hard, and while convergence may be achievable in static scenarios, this approach falls short when addressing the dynamic nature of real-world conditions, particularly due to the significant computational overhead involved in recalculating solutions as environmental conditions change over time. In high-density vehicular networks, the optimization variables are highly coupled, and the phase shift selection for RICS becomes increasingly complex due to varying conditions at different locations and times, as well as differing channel states. To overcome these limitations, some studies have considered integrating deep reinforcement learning (DRL) into the optimization of wireless communications~\cite{DRL,DRL1,DRL2,DRL3,DRL4}. The research in \cite{DRL01} delves into the use of DRL to optimize drone trajectories and beamforming. Additionally, \cite{DRL02} studies the problem of the enhancement of V2I communication using an RIS in vehicular edge computing systems and addresses optimization problems through reinforcement learning methods. Besides, \cite{DRL03} explores a STAR-RIS-assisted V2X communication system that combines STAR-RIS with DRL algorithms. Additionally, \cite{DRL04} utilized DRL with Lyapunov optimization to investigate the trade-offs among computation, communication, and latency in RIS-enabled MEC systems. Building upon this literature, \cite{DRL05} further employed a scalable multi-agent DRL (MADRL) framework for optimizing user scheduling and precoding in distributed RIS-aided communication systems. However, little progress has been made in applying DRL for RIS-integrated time-varying V2V/V2I systems.

\section{RICS-Assisted AV Systems}
In this section, we present the RICS-assisted autonomous driving system model and the channel model used in this paper. We also introduce the computation model for the RICS and our design optimization formulation. 

\subsection{Network Model}
We investigate an uplink autonomous driving scenario facilitated by RICS, as illustrated in Fig.~\ref{scene2}. This scenario consists of a base station (BS) that simultaneously serves multiple cells, which can be denoted as ${\cal C}=\{1, 2, \ldots, C\}$. Each cell is equipped with a RICS that has 
$K$ reflecting elements and $A$ AVs that communicate with the BS via Vehicle-to-Infrastructure (V2I) links, while also sharing information with several Vehicle-to-Vehicle (V2V) pairs. The sets of AVs and V2Vs can be represented as follows: ${\cal U}=\{1, 2, \ldots, U\}$ and ${\cal V}=\{1, 2, \ldots, V\}$. In this system, AVs have the capability to transmit computing tasks to the MEC server through V2I links, as well as share information with other vehicles using V2V links. 

To accommodate the simultaneous information transmission needs of multiple vehicles to the BS for computational assistance, a Frequency Division Multiple Access (FDMA) approach is employed \cite{FDMA}. Each AV uses its allocated sub-channel for information transmission, thereby avoiding mutual interference. The total bandwidth is denoted as $W$, and the $U$ AVs share the channel equally. We assume that the total transmission and computation time for AVs across multiple cells is divided into $L$ equal, non-overlapping time slots. During each time slot, AVs occupy their sub-channels following the FDMA protocol. To enhance spectrum utilization, V2V pairs may share the sub-channel of a specific AV for information transmission. Spectrum sharing is represented by a binary variable $\omega_{u,v}$, where $\omega_{u,v}$ indicates channel sharing between the $u$-th AV and the $v$-th V2V, causing interference to the V2V pairs. In each time slot, we assume that both the AVs and the V2Vs remain stationary. After completing the algorithm during this time slot, they will move to their positions in the subsequent $(l+1)$-th time slot.

Furthermore, The specific structural configuration of the RICS is referenced in \cite{RICS_tits}. To put it simply, the first layer operates in reflection-refraction (RR) mode, which reflects a portion of the incident signal, with the reset facilitating signal refraction to support V2V users on the opposite side, thus extending the signal coverage. The second layer works in AC mode, utilizing metamaterials to adjust the amplitude of the transmitted signals, thereby generating signals that can interfere destructively with V2I interference waves to achieve interference cancellation.
Specifically, the energy splitting ratio on the reconfigurable intelligent surface is defined as ${\chi_k} \!\triangleq\! \beta_k^{r}\!:\! \beta_k^{t}, \ \forall k \in {\cal K} \triangleq \{1, 2, \ldots, K\}$, where $\beta_k^{r}$ and $\beta_k^{t}$ both in $[0,1]$ and $\beta_k^{r}+\beta_k^{t}=1$. The amplitude adjustment factor for the $k$-th element of the intelligent computing layer on transmitted signals is denoted as ${\Psi}_k$. The reflection-refraction signals can then be expressed as: $s_k^{r} (l)= \sqrt{\beta_k^{r}}e^{j \theta_k^r(t)}s_k$  and $s_k^{t} (l)= \sqrt{\beta_k^{l}}e^{j \theta_k^t(l)}s_k$, with $\theta_k^r(l),\theta_k^t(l)\in(0,2\pi]$. The refraction-reflection coefficient matrix for the $i$-th RICS are given by ${\mathbf{\Theta}}_r(l)={\rm{diag}} \{s_1^{r}(l), s_2^{r}(l), \ldots, s_K^{r}(l) \}$ and ${\mathbf{\Theta}}_t(l)={\rm{diag}} \{s_1^{t}(l), s_2^{t}(l), \ldots, s_K^{t}(l) \}$, respectively.

\begin{figure}[t]
	\centering
	\includegraphics[width=1.1\columnwidth]{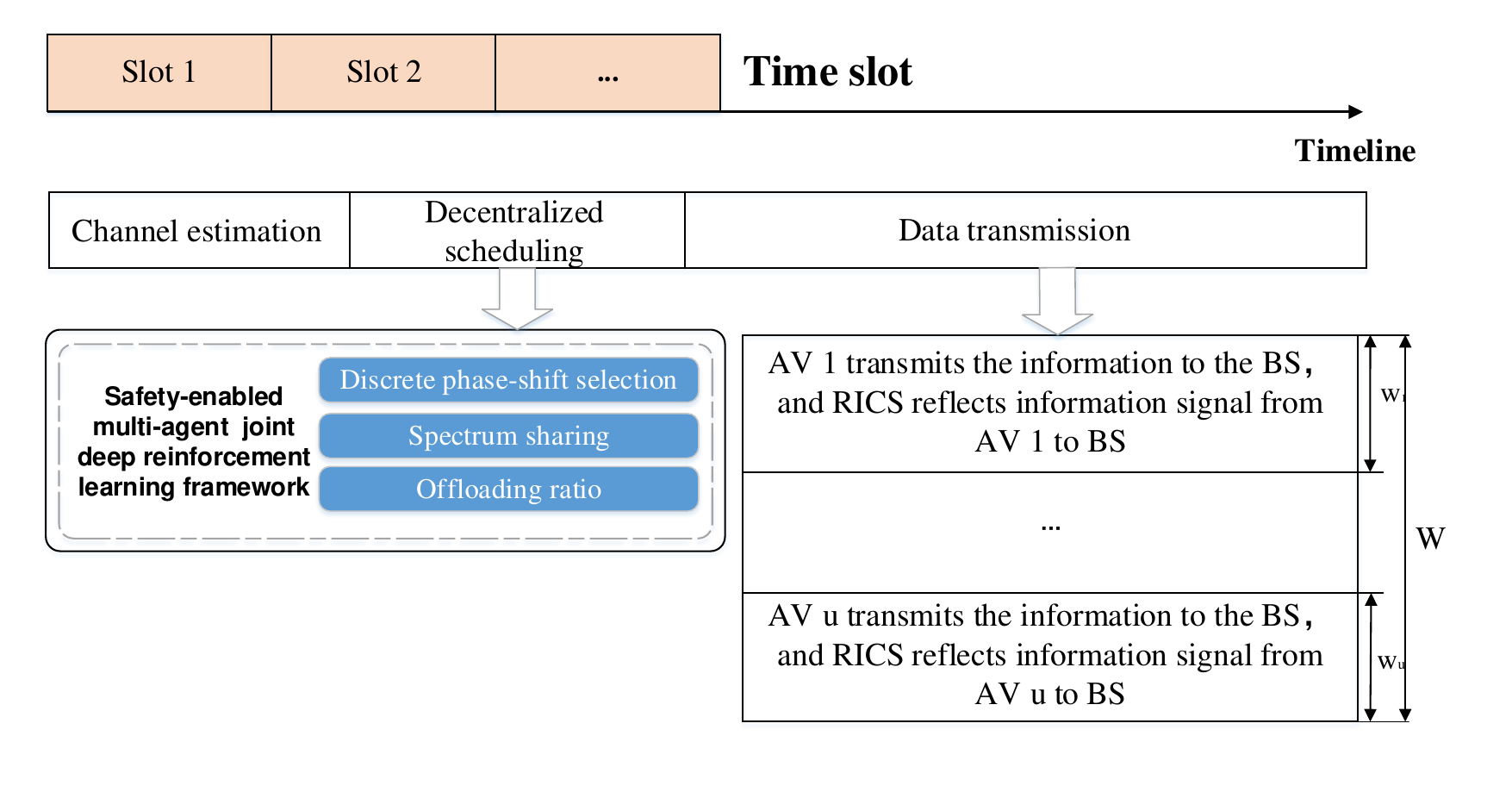}
	\caption{The proposed DS-MADRL mechanism.}
    \label{timeslot}
\end{figure}

\subsection{Channel Model}
In the proposed $c$-th cell of the RICS-assisted autonomous driving system, we define ${\textbf h}_{u,r} \in \mathbb{C}^{K \times 1}$, ${\textbf h}_{r,b} \in \mathbb{C}^{1 \times K}$, and ${\textbf h}_{r,v} \in \mathbb{C}^{K \times 1}$ to represent the channels from the $u$-th AV to the RICS, from the RICS to the BS, and from the RICS to the $v$-th V2V pair, respectively. The modeling of wireless channels takes into account both large-scale and small-scale fading components. Large-scale fading is primarily determined by distance, which can be represented by $P(l)=\sqrt{{{C}_{0}}{{\left( \frac{d(l)}{{{d}_{0}}} \right)}^{\alpha }}}$ \cite{channel}. Small-scale fading, on the other hand, is influenced by the environment, including the line-of-sight and Non-line-of-sight (NLoS). We denote ${\textbf h}^{Los}$ to represent the direct link, which is closely related to the placement angles of the antennas. In this system, the RICS adopts a Uniform Linear Array (ULA) configuration, and therefore the LoS channel can be expressed as ${\textbf{h}^{Los}}={{\alpha }_{T}}{{\alpha }_{R}}$, where $\alpha_t$ and $\alpha_r$ represent the angular vectors and the transmitting and receiving antennas. NLoS consists of multiple reflected paths superimposed upon the direct link, typically following a circular complex Gaussian distribution, i.e. $\textbf{h}^{NLoS} \sim \mathcal{CN}(0, 1)$. Thus, the channel gains can be summarized as follows
\begin{equation}\label{h_mr}
{{\textbf{h}}_{u,r}}\left( l \right)=P(l) \left( \sqrt{\frac{{{\zeta }_{u,r}}}{1+\zeta_ {u,r}}}\textbf{h}_{u,r}^{Los}+\sqrt{\frac{1}{1+{{\zeta }_{u,r}}}}\textbf{h}_{u,r}^{NLoS} \right),
\end{equation}

\begin{equation}\label{h_rn}
{\textbf h}_{r,v}(l)=P(l)\left( \sqrt{\frac{{{\zeta }_{r,v}}}{1+\zeta_{r,v}}}\textbf{h}_{r,v}^{Los}+\sqrt{\frac{1}{1+{{\zeta }_{r,v}}}}\textbf{h}_{r,v}^{NLoS} \right).
\end{equation}
where, ${{\zeta }_{u,r}}$, ${{\zeta }_{r,v}}$ are denoted as Rayleigh factors in small-scale fading. In addition, we define the channel of the direct link as: $h_{u,b}$ and $h_v$, which denote the channel gains from the $u$-th AV to the BS and from the transmitter to the receiver of the $v$-th V2V pair, respectively. The interference link has $h_{u,v}$ and $h_{v,b}$ denoting the interference channel gains from the $u$-th AV to the Rx of the $v$-th V2V pair, respectively. We assume that both BS and vehicles have perfect channel state information (CSI), which will be preprocessed at the beginning of each time slot for channel estimation.

Based on the above analysis, the SINR received for the $u$-th AV and the $v$-th V2V pair can be derived as (\ref{sinr_m})(\ref{sinr_n}), where $P_u$ and $P_t$ represent the transmission power of the AV and the Tx of the V2V pair, respectively. Therefore, the achievable uplink rate for the $u$-th AV and $v$-th V2V is given by (\ref{Ru})(\ref{Rv}) Similarly, the achievable data rate for the $v$-th V2V pair can be obtained.

\begin{equation} \label{sinr_m}
\gamma_{b}^u(l) = \frac{{P_u}(l) \left |{h}_{u,b}(l) +  {\textbf h}_{r,b}(l) \mathbf{\Theta}_r(l) {\textbf h}_{u,r}(l)  \right |^2 }{\sum\limits_{v=1}^{V}{\omega_{u,v}(l) P_t(l)} \left |{\textbf h}_{v,b}(l)   \right |^2 + W \xi_0},
\end{equation}
\begin{equation} \label{sinr_n}
\begin{split}
    &\gamma_v(l)=\\
    &\frac{{P_t}(l) \left |h_{v}(l) \right |^2 }{\sum\limits_{u=1}^{U}{\omega_{u,v}(l) P_u(l)}  \big|h_{u,v}(l) \!+\! {\textbf h}_{r,v}^H(l) \mathbf{\Theta}_t(l) {\textbf h}_{u,r}(l)\big|^2 \!+\! W\xi_0},
\end{split}
\end{equation}

\begin{equation} \label{Ru}
   R_{b}^u(t)=\frac{W}{u} {\rm log}_2{\left ( 1+ \gamma_{b}^u(l) \right )}.
\end{equation}
\begin{equation} \label{Rv}
   R^v(t)=\frac{W}{v} {\rm log}_2{\left ( 1+ \gamma_v(l) \right )}.
\end{equation}

\subsection{RICS Computation Model}
In RICS-assisted autonomous driving networks, real-time sensor fusion and computation offloading are crucial for ensuring driving safety. Sensor fusion integrates data from multiple sensors to provide a more comprehensive and accurate understanding of the vehicle's surrounding environment. Each vehicle is required to perform a series of perception and decision-making tasks. Furthermore, the proposed partial offloading model allows vehicles to offload part of the task to the MEC server. This approach reduces the CPU usage of onboard devices, thereby lowering energy consumption and improving system response times to ensure safe driving.

To simplify the model, we describe the tasks executed by the AVs using three variables: $s_{u,c}$[bits] represents input data size for computation, $\varsigma_{u,c}$ [cycles] indicates the CPU cycles required to process $s_{i,j}$, and $\sigma_{u,c}$ [secs] denotes the maximum allowable delay.
To improve computational efficiency, the models deployed at BS, and the vehicles are configured according to their hardware capabilities, ensuring fast response times and high-precision environmental perception. When processing image data of a specific quality $q$, the AVs DNN-based inference accuracy is guaranteed no larger than the BS's inference accuracy, i.e., $A_{u,c}(q) = \lambda A_{b}(q)$, $0 \leq A_{u,c}(q), A_{b}(q),\lambda \leq 1$. 

After providing the above definitions, a partial offloading model is designed to fully utilize computational resources. Specifically, the offloading ratio $\rho_u$ represents the ratio of tasks offloaded to the BS, with $1-\rho_u$ indicating the data ratio that needs to be processed locally. It is important to note that the tasks being partially offloaded by vehicles are divisible video sequence tasks. This is because a video sequence can be decomposed into multiple independent frames or segments, each of which can be processed independently \cite{part_offload}. Therefore, each task on an AV can be divided into two parts: the local computation delay of $u$-th AV denoted as 
\begin{equation}
    \tau^{u,c}_{loc}(l) = \left ( 1-\rho_{u,c}(l) \right )  \frac{\varsigma_{u,c}}{f_{u,c}}.
\end{equation}
Considering the scenario where multiple cells share a single BS, the delay in computation offloading considers that multiple cells may simultaneously upload tasks to the BS for processing. In this case, a resource allocation strategy is employed where the BS evenly distributes computational resources among multiple cells. Since the tasks from each cell are processed in parallel, the overall computation delay is determined by the cell that requires the most time to complete. Thus, the computation offloading delay can be expressed as follows:
\begin{equation}
    T_{off}^{u,c}\left( l \right)={{\rho }_{u,c}}\left( l \right)\left( \frac{{{s}_{u,c}}}{R_{b}^{u,c}\left( l \right)}+{{\mathop{\max }}}\,\left( \frac{{{{\varsigma}_{u,c}}}}{F_{u,c} }\right) \right),
\end{equation}
where $f_{u,c}$ represents the computational resources of the AVs, while $F_{u,c}$ indicates the resources allocated by the BS for the computational tasks assigned to each AV. Here, the BS distributes the task resources evenly among all vehicles.
As a result, the total delay of a task on the $u$-th AV is calculated as $\tau_{u,c}(l)={\rm max} \{\tau^{u,c}_{loc}(l),  \tau^{u,c}_{off}(l)  \}$. Based on this, the average inference accuracy of tasks can be obtained
\begin{equation}
    \tilde{A}_{u,c}(q)= \left ( 1- \rho_{u,c}(l) \right )A_{u,c}(q) + \rho_{u,c}(l) A_{b}(q).
\end{equation}

\subsection{Problem Formulation}
Using the latter expressions, We define the driving safety factor as follows:
\begin{equation}\label{saftey}
{S}_{u,c}(l) = \frac{\tilde{A}_{u,c}(q)}{\tau_{u,c}} = \frac{\left ( 1- \rho_{u,c} \right )A_{u,c}(q) + \rho_{u,c} A_{b}(q)}{{\rm max} \{\tau^{u,c}_{loc},  \tau^{u,c}_{off}  \}} .
\end{equation}
Our goal is to maximize the sum safety factor of the AVs while satisfying the outage probability of V2V pairs, i.e., in mathematical terms:
\begin{subequations} \label{binary_original_problem}
\begin{align}
& \;\;\;\; \mathbb{P}: \ \underset{\bm{\omega}, \mathbf{\Theta}_x, \bm{\rho}}{\mathop{\max }}\,\;\;\;\frac{1}{L}\sum\limits_{l}^{L}{\sum\limits_{c}^{C}\sum\limits_{u}^{U}{{{S}_{u,c}}}}\left( l \right)  \notag \\ 
& \;\;\;\;\;\;{{s}}{{.t}}{\rm{.}}\;\;\; \; {\rm Pr}\{\gamma_v(l) \leq \gamma_{th}\} \leq P_{outage}(l),\;\forall l , \\
&\;\;\;\;\;\;\;\;\;\;\;\;\;\; \; \omega_{u,v}(l) \in {0,1}, \;\;\;\;\;\;\;\;\;\;\forall u , \forall v ,\forall l ,  \\
& \;\;\;\;\;\;\;\;\;\;\;\;\;\;  \sum_{v=1}^{V} \omega_{u,v}(l) \le 1, \ \;\;\;\;\;\;\forall u , \forall v ,\forall l ,  \\
& \;\;\;\;\;\;\;\;\;\;\;\;\;\; \; 
\beta_k^{r}+\beta_k^{t}=1, \;\;\;\;\;\;\;\;\;\;\ 1 \le k \le K, \\
& \;\;\;\;\;\;\;\;\;\;\;\;\;\; \; \rho_u(l) \in [0,1], \;\;\;\;\;\;\;\;\;\;\; \forall u ,
 \end{align}
\end{subequations}
where $\bm \omega = \{\omega_{u,v}, \forall u, v\}$ and $\bm \rho=\{\rho_1, \rho_2, \ldots, \rho_u\}$.

\subsection{Outage Probability of V2V Pairs}
In this section, we address the outage probability constraint of V2V pairs in (\ref{binary_original_problem}a) by approximating it using a smooth step function $\hat{u}_\delta(x)=\frac{1}{1+e^{-\delta x}}$ \cite{step}, which includes a smoothing parameter $\delta$ to control the approximation error. Subsequently, we can approximate the constraint (\ref{binary_original_problem}a):
\begin{equation} \label{proof01}
\mathbb{E} \left [\hat{u}_\delta \left(\gamma_{th}- \gamma_v(l)\right)  \right ] \leq P_{outage}(l).
\end{equation} 

With the approximated constraint provided in (\ref{proof01}), we can calculate the outage probability in constraint (\ref{binary_original_problem}a), leading to a further simplification, as outlined in \textbf{Theorem~\ref{Tm1}}.

\newtheorem{theorem}{\textbf{Theorem}}
\begin{theorem}
\label{Tm1}
Let $\tilde{\gamma}_v (\bm{\omega}(l), \mathbf{\Theta}_x(l), \bm{\rho}(l)) = \mathbb{E} \left [ \gamma_v \right ]$, we can represent the constraint  (\ref{binary_original_problem}a) as
\begin{equation} \label{proof02}
\tilde{\gamma}_v (\bm{\omega}(l), \mathbf{\Theta}_x(l), \bm{\rho}(l)) \geq \gamma_{th}+\frac{1}{\varpi} {\rm ln} \left ( \frac{1}{P_{outage}(l)}-1 \right)\triangleq \tilde{\gamma}_c(l).
\end{equation}
\end{theorem} 

\begin{proof}
Please refer to \cite{RICS_tits}.
\end{proof}
 
Based on {\textbf{Theorem~\ref{Tm1}}}, we can conveniently reformulate constraint (\ref{binary_original_problem}a) to address the optimization problem $\mathbb{P}$, leading to the following equivalent problem:
\begin{align} \label{equivalent_problem}
& \;\;\;\; {\mathbb{P}}: \ \underset{\bm{\omega}, \mathbf{\Theta}_x, \bm{\rho}}{\rm max}  \;\;\;\frac{1}{L}\sum\limits_{l}^{L}{\sum\limits_{c}^{C}}{{\sum\limits_{u}^{U}{{{S}_{u,c}}}}}\left( l \right)   \\
& {{s}}{{.t}}{\rm{.}}\;\;\;\;\;\;\; (\ref{binary_original_problem}b) - (\ref{binary_original_problem}e), (\ref{proof02}). \notag
 \end{align}
Note that this problem is non-convex according to \cite{RICS_tits}. Additionally, a strong coupling among the three optimization variables complicates the problem solution further. Traditional algorithms often use alternating optimization techniques to approach the optimal solution; however, these methods struggle to handle time-varying channels. Motivated by this limitation, we develop a DS-MADRL algorithm to find a viable solution for $\mathbb{P}$.

\section{RICS-Assisted Driving Safety Maximization}
In this section, the proposed DS-MADRL scheme, its training, and algorithmic steps are introduced. The section commences with a brief description of MADRL principles.

\subsection{Markov Game Formula}
RL is an important branch of ML, which learns the optimal strategy to maximize long-term reward through the interaction between the agent and the environment. DRL integrates the decision planning ability of RL and the feature learning ability of DL, so that the agent can process high-dimensional and complex input data and learn more complex strategies \cite{pre_DRL}.

In the previous section, we modeled the optimization problem (\ref{equivalent_problem}) as a continuous decision process over a time frame, i,e, the next time slot will make the decision based on the current state. The agent aims to maximize the long-run expected discount reward by learning an optimal strategy $P$. This learning is based on interaction with the environment, which is modeled as a Markov decision process (MDP) denoted by a quintuple $\left( \cal{S},\cal{A},\cal{R},\cal{T},\gamma  \right)$. This representation encompasses the state space, the action space, the reward function, the transport strategy, and the discount factor. However, in the complex scenario of this paper, a single agent may be affected by other individuals, and it is difficult to learn the overall strategy. Therefore, we introduce multiple agents here, and the corresponding Markov process also becomes the Markov game (MG). In MG, multiple agents interact in a shared environment, and each agent can choose its own actions to affect the environment. In contrast to MDP, MG considers the competition and cooperation between multiple agents. Correspondingly, MG can also be described by a quintuple:  $\left( \cal{N},\cal{S},\cal{A},\cal{P},\cal{R}  \right)$, where $\cal{N}$ is the number of agents involved and $\cal{A}$ is the Cartesian product of all agents' action spaces. $\cal{P}$ denotes the transition probability from one state to another, and $\cal{R}$ indicates the reward function of each agent, that is, the immediate reward function obtained by each agent under a given state and action. The long-term discounted reward in the multi-agent setting MG represents the goal that each agent focuses on when selecting actions. This reward takes into account future uncertainty and discounts future rewards to ensure that the impact of distant rewards on the behavior of agents is not overly overlooked.

In MG, each agent selects actions that maximize its expected long-term cumulative discounted reward. By considering the impact of discount factors and future rewards, the agent can strike a better balance between immediate and long-term expected rewards, leading to more astute decisions and action selections. To achieve the maximum long-term cumulative discounted reward, in time slot $l$, our long-term cumulative discounted reward is given by:
\begin{equation}
    {{G}_{l}}=\sum\limits_{i=0}^{L }{{{\gamma }^{i}}{{R}_{l+i}}}.
\end{equation}
In this context, the discount factor $\gamma \in \left[ 0,1 \right]$ signifies the importance of future rewards to the agent. A higher value of $\gamma$ near 1 indicates that the agent places greater emphasis on long-term cumulative rewards. Conversely, a lower value of $\gamma$  closer to 0 suggests that the agent prioritizes immediate rewards.

Based on the above model, the MG elements of the RICS-assisted autonomous driving vehicular network system are described as follows.
\begin{enumerate}
    \item \textit{State space}: The state space is the collection of system information at a specific time point, used to guide decision-making. It consists of the local states of RICS and AVs. For the RICS, the state space is defined by considering the historical actions taken by the system along with their corresponding channel states, which provide context for future decisions.  For the AVs, the state is characterized by the current power level of each vehicle and the prevailing channel state, which reflects the vehicle’s operational status and its ability to communicate effectively with the RIS and other entities in the network. The specific modeling of the state space for both the RIS and AVs is detailed as follows
    \begin{equation}
        {s}_{RICS}\triangleq \left[ {{a}_{k}}\left( l-1 \right),{{\textbf h}_{u,r}}\left( l \right),{{\textbf h}_{r,v}}\left( l \right),{{\textbf h}_{r,b}}\left( l \right) \right],
    \end{equation}
    \begin{equation}
        {s}_{\mathcal{U}}\triangleq \left[ {{P}_{u}}\left( l \right),{{P}_{l}}\left( l \right),{{a}_{u}}\left( l-1 \right),{{h}_{u,v}}\left( l \right),{{h}_{u,b}}\left( l \right) \right].
    \end{equation}
    The global state set composed of the local states of all agents is defined as:
    \begin{equation}
        \mathcal{S}_l\triangleq \left[ { {{s}_{RICS} }},{{s}_{\mathcal{U}}} \right].
    \end{equation}
    \item \textit{Action space}: The action space mainly describes the set of all actions that the system can take. An action is the course of action taken by the system at a specific time point, used to alter its state. Due to the specific nature of the RICS structure, the amplitude coefficients of its refraction-reflection coefficient matrix and the amplitude adjustment factor of AC mode are limited by the material properties and cannot be adjusted over time. Therefore, we do not optimize them, and for the amplitude coefficients, we adopt a balanced transmission and reflection mode, i.e., ${{\beta }_{t}}={{\beta }_{r}}=0.5$, evenly distributing energy between transmission and reflection. In addition, if each element of RICS is independently controlled, a large number of parameters are required, which would increase the training overhead. To address this issue, we partition the RICS into Q sub-blocks, assigning the same phase shift values to the elements within each sub-block. The feasible domain of discrete phase shift adjustments for the sub-blocks is as follows
    \begin{equation}
        \theta _{l}^{t},\theta _{l}^{r}\in \left\{ 0,\frac{2\pi }{{{2}^{h}}},\ldots,\frac{2\pi \left( {{2}^{h}}-1 \right)}{{{2}^{h}}} \right\},
    \end{equation}
    $h$ represents the resolution bits. The corresponding RICS action space is defined as follows:
    \begin{equation}
        {{a}_{RICS}}\triangleq \left[ \left\{ \theta _{l}^{t},\theta _{l}^{r} \right\} \right].
    \end{equation}
    The local action spectrum sharing strategy $\omega$ of AVs is a discrete action, corresponding to (\ref{binary_original_problem}b). The offload ratio $\rho$ is a continuous action, corresponding to (\ref{binary_original_problem}e).  define their action space as:
    \begin{equation}
        {{a}_{u}}\triangleq \left[ {{\omega }_{u}},{{\rho }_{u}} \right].
    \end{equation}
    Therefore, the local actions of all agents constitute joint actions:
    \begin{equation}
        {\mathcal{A}_{l}}\triangleq \left[ {{a}_{RICS}},{{a}_{u}} \right].
    \end{equation}
    
    \item \textit{Reward}: The reward function is the immediate reward obtained after the system takes a specific action. Based on the optimization problem (\ref{equivalent_problem}), we define it as:
    \begin{equation}
        \begin{split}
            r\triangleq\underbrace{\sum\limits_{u=1}^{U}{{{S}_{u}}}}_{part1} +\underbrace{  \sum\limits_{v=1}^{V}{\min \left\{ {{\tilde{\mathop{\gamma }}\,}_{v}}\left( \bm{\omega} ,{{\bm{\Theta} }_{x}} \right)-{{\tilde{\mathop{\gamma }}\,}_{c}},0 \right\}}}_{part2}.
        \end{split}  
    \end{equation}
    
    The reward function that guides the learning process is consistent with the objectives of multi-objective optimization. To achieve the goal of maximizing the total safety index of AVs, subject to the constraints (\ref{binary_original_problem}c), we introduce a penalty, and if any of these constraints are not satisfied, the constraint set is terminated.
    \begin{equation} \label{constraint}
        R\left( l \right)=\left\{ \begin{aligned}
        & -Penalty,\text{      if  }{{S}_{u}}=NS \\ 
        & r\left( l \right),\text{    otherwise,} \\ 
        \end{aligned} \right.
    \end{equation}
    where $NS$ represents a negative state, indicating that when the state of a certain AV fails to satisfy the constraint (\ref{binary_original_problem}c). The state here includes the actions taken by the agent in the previous stage, and if the action violates the constraints, a penalty is imposed. In other words, the algorithm receives a negative reward $Penalty$, where $Penalty$ is a sufficiently large positive number used to ensure that the algorithm quickly identifies and avoids invalid paths when encountering unsatisfied conditions, thereby improving its convergence speed.
\end{enumerate}

\begin{figure*}[!t]
\centering
\includegraphics[width=1.8\columnwidth]{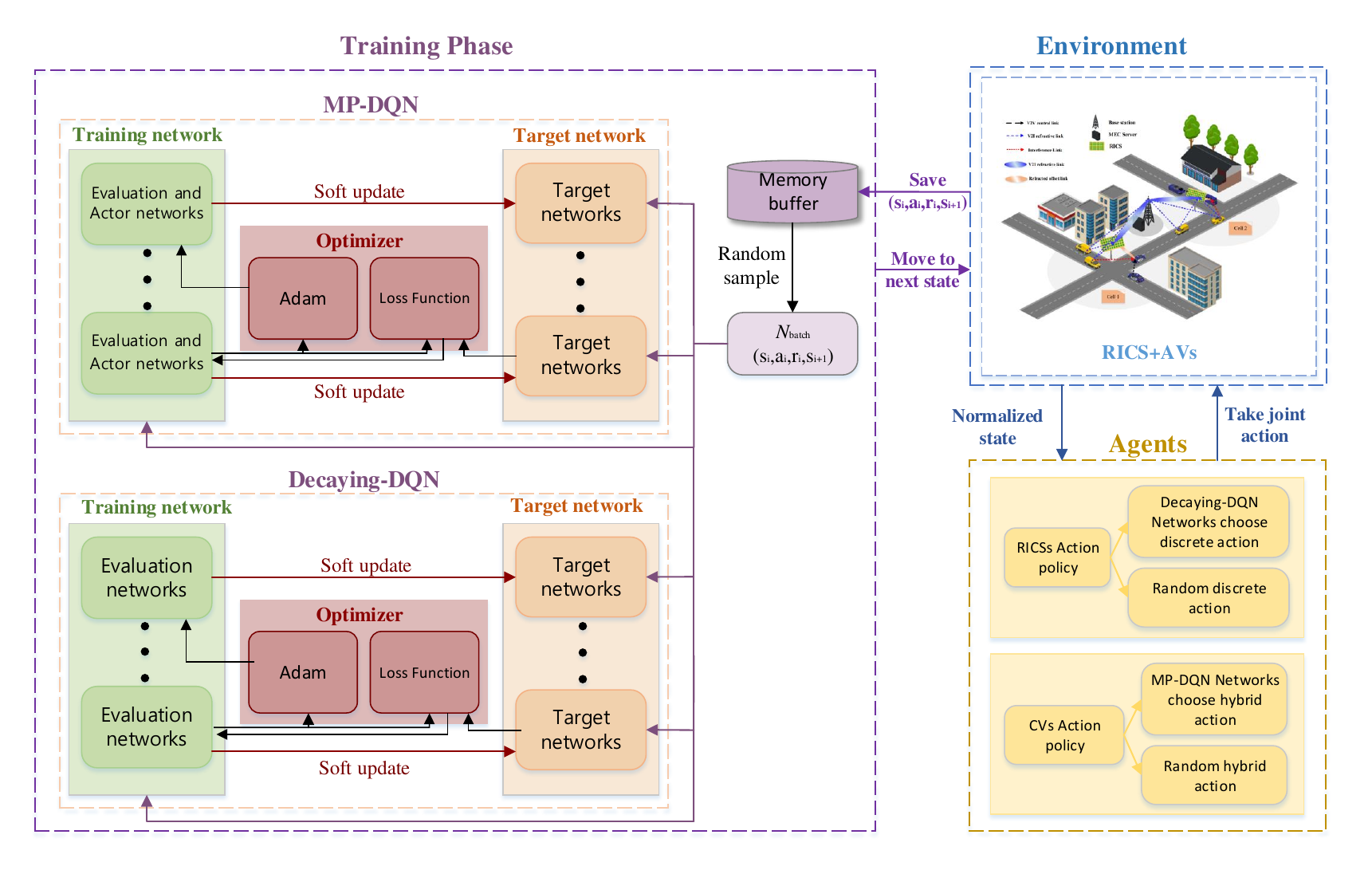}
\caption{The proposed training framework for our DS-MADRL algorithm.}
\label{training_frame} 
\end{figure*}

\subsection{DS-MADRL Network Architecture}
For the RICS-assisted autonomous driving vehicular system model presented in this paper, we design DS-MADRL to accommodate two types of agents. This framework aims to integrate the action selection tasks of various types of agents through joint decision-making and interaction with the environment, ultimately enhancing system safety.
Considering the different types of actions performed by the two kinds of agents, we incorporate DDQN and MP-DQN networks to address the discrete phase shift selection problem of RICS, as well as the spectrum sharing strategy and task offloading ratio. The introduction of DDQN is specifically intended to tackle the discrete phase shift selection issue for RICS. Since the phase shifts for RICS are discrete, DDQN learns the mapping relationship between phase shifts and channel conditions, enabling it to select the optimal phase shift values to enhance data transmission rates.
On the other hand, since AVs must simultaneously manage the task offloading ratios and spectrum-sharing strategies, we model this decision-making process as a mixed continuous-discrete action space. The MP-DQN network is well-suited to effectively learn in such hybrid action spaces, ensuring the maximization of system safety across different environments.
We anticipate that this joint learning framework will facilitate effective collaborative decision-making in complex interactions and dynamic environments, thereby improving the overall system performance.

Deep Q-Network (DQN) is a reinforcement learning algorithm that combines Q-learning and deep learning, aiming at solving decision-making problems in discrete high-dimensional state spaces. DQN uses a neural network to estimate the state-action value function, i.e., $Q\left( s,a \right)={{\mathbb{E}}_{\pi }}\left[ {{G}_{t}}|{{s}_{t}}=s,{{a}_{t}}=a \right]$, which we call the Q value. The goal is to find the optimal strategy ${{\pi }^{*}}=\arg {{\max }_{a}}{{Q}^{*}}\left( s,a \right)$ to maximize this Q value. The DQN structure includes a evaluation Q network to estimate the state-action value function, which maximizes $Q(s,a)$ by selecting action $a$, which we define as: ${{Q}^{*}}\left( s(l),a(l);\delta_w \right)={\mathbb{E}}\left[R(l)+\gamma {{\max }_{a_{(l+1)}}}{{Q}}\left( s(l+1),a(l+1) \right)\right]$, where $\delta$ represents the parameter of the Q network. And a target Q network with the same structure as it is used to calculate the target ${{Q}^{*}}\left( s,a \right)$ value. During the training process, the parameters of the $Q_{tar}$ network are not immediately updated at each training step, but instead, the update is delayed for a certain period. These two neural networks are updated by minimizing the loss function, as shown in equation (\ref{lossDQN}), and the evaluation of the state-action values is updated using the Bellman equation, as described by $Q\left( s(l),a(l) \right)={\mathbb{E}}\left[R(l)+\gamma {{\max }_{a_{(l+1)}}}{{Q}_{tar}}\left( s(l+1),a(l+1) \right)\right]$. Through the evaluation of the Q-network, the appropriate strategy for DQN is to select a policy that maximizes the state-action values.
\begin{figure*}[htb]
\begin{equation} \label{lossDQN}
\mathcal{L}\left( \delta_w \right)={{\left( R(l)+\tau {{\max }_{a{(l+1)}}}Q\left( s(l+1),a(l+1);{{\delta}^{-}_w} \right)-Q\left( s(l),a(l);\delta_w \right) \right)}^{2}}
\end{equation}
\begin{equation} \label{bellman}
    Q\left( s\left(l\right),a\left(l\right),{{x}_{a}\left(l+1\right)} \right)=
    {\mathbb{E}}\left[R(l)+\gamma \underset{{{a}{(l+1)}}}{\mathop{\max }}\,\underset{{{x}_{a}}(l+1)}{\mathop{\sup }}Q\left( s(l+1),a(l+1),x_{a}{(l+1)} \right)\right]
\end{equation}
\rule{\textwidth}{1pt}
\end{figure*}
In order to make the training more stable and quickly converge to the optimal strategy, we introduce an improved DQN algorithm, decaying DQN, which introduces a Q-decay mechanism in the training process. It can balance accelerated convergence and avoid falling into the local optimal solution during the training process. The formula for the decaying learning rate is expressed as follows:
\begin{equation}
    \alpha \left( episodes \right)={{\alpha }_{0}}\left( 1/\left( 1+\varepsilon \times episodes \right) \right),
\end{equation}
where $\alpha$ denotes the learning rate, a variable related to the number of training episodes, $\alpha_0$ denotes the initial learning rate, and $\varepsilon$ denotes the learning rate decay.

Parametrized DQN (P-DQN) is a network that can handle continuous and discrete action spaces and consists of two main networks: the Q network and the actor-network. Specifically, the Q network is used to receive the state and the joint action parameters ${s,x}$ as inputs, and outputs the Q-value corresponding to each discrete action. In contrast, the actor-network receives the state as an input and outputs the best continuous action corresponding to each discrete action. However, this network suffers from a critical problem: it inputs all action parameters jointly into the Q network, which may result in each Q value $Q_i$ being a function of all action parameters $x$, and not just the continuous action parameter $x_i$ associated with that discrete action, which can lead to problems of ineffective gradients and suboptimal action selection. The MP-DQN solves this problem by separating action parameters through multiple passes to separate the action parameters.

In MP-DQN, the Bellman equation is redefined as (\ref{bellman}) to accommodate cases with both continuous and discrete action spaces. For each discrete action a, $x_{a}^{*}$ is obtained by calculating $x_{a}^{*}=\arg {{\sup }_{x_{a}{(l+1)}}}Q\left( s(l+1),a(l+1),x_{a}^{*} \right)$, and then $a^*$ is obtained by calculate ${{a}^{*}}=\arg {{\max }_{a(l+1)}}Q\left( s(l+1),a(l+1),x_{a}^{*} \right)$. The MP-DQN structure uses deterministic policy network ${{x}_{a}}\left( s;{{\delta }_{x}} \right)$ to approximate continuous action $x_{a}^{Q}=\arg {{\sup }_{{{x}_{a}}}}Q\left( s,a,{{x}_{a}} \right)$, and the evaluation Q network $Q\left( s,a,{{x}_{a}};{{\delta }_{q}} \right)$ is used to approximate the state-action value function $Q\left( s,a,{{x}_{a}} \right)$. In addition, there is a corresponding target $Q\left( s,a,{{x}_{a}};\delta _{q}^{-} \right)$ network and target policy network ${{x}_{a}}\left( s;\delta _{x}^{-} \right)$. MP-DQN works similarly to DQN by interacting with the environment and storing experiences in the experience replay buffer. When training, it randomly selects a set of tuples from the experience replay buffer to train to reduce the correlation between observations and decisions. In addition, the target Q value of MP-DQN at the $l$ training step is as follows:
\begin{equation}
    y\left( l \right)=R(l)+\tau \max Q\left( s\left( l+1 \right),a,{{x}_{a}}\left( s\left( l+1 \right);\delta _{x}^{-} \right);\delta _{q}^{-} \right).
\end{equation}

\begin{algorithm}[t]
 \SetAlgoLined
 \small
 \KwResult{Sum safety factor $S_u$, sum V2V data rate $R_v$, offloading policy $\phi_{\rho}$, spectrum sharing policy $\phi_{\omega}$, and refraction reflection coefficient matrix policy $\phi_{\Theta}$}.
 Initialize all DDQN agents networks $Q(s,a;\delta_w)$ and $Q(s,a;\delta^-_w)$ with $\delta_w=\delta^-_w$, and also initialize all MP-DQN agents' networks $Q(s,a,x_a;\delta_q)$, $x_a(s;\delta_x)$, $Q(s,a,x_a;\delta^-_q)$, $x_a(s;\delta^-_x)$ with $\delta_q=\delta^-_q$  and  $\delta_x=\delta^-_x$;
 
 \For{$e=1,2,\ldots,E$}{
    Reset vehicle positions and generate initial state $s(0)$\;
    \For{$t=1,2,\ldots,T$}{
        \For{each RICS sub-blocks $q=1,\ldots,Q$}{
            Each $q$-th DDQN agent selects a discrete phase shift $a_2$ using $\epsilon$-greedy algorithm\;
        }
        \For{each AV $m=1,2,\ldots,M$}{
            Each $u$-th MP-DQN agent selects a joint action $a_1$ using $\epsilon$-greedy policy\;
        }
        The environment executes joint action $a(l)$, obtains the reward $r(l)$ based on (\ref{constraint}) and transitions to the next state\;
        Store tuple $(s(l),a(l),r(l),s(l+1))$ in the experience replay buffer\;
        Sample a mini-batch of size $B$ from the experience replay buffer\;
        
        Update weights $\delta_w$ for all DDQN agents by minimizing the loss function according to (\ref{lossDQN}) and update weights $\delta^-$ by copying $\delta_w$\;
        
        Update weights $\delta^-_q$ and $\delta^-_x$ for all MP-DQN networks based on (\ref{lossq})(\ref{lossx})\;
        
        Additionally, update the target networks using a soft replacement approach\;
    }
 }
 \caption{DS-MADRL Training for $\mathbb{P}$ in (\ref{equivalent_problem})} \label{a1}
\end{algorithm}

Next, we update the parameters of the Q network and the policy network by minimizing the loss function, specifying the loss function of the deterministic policy network, and evaluating the Q-network given by
\begin{equation} \label{lossx}
    \mathcal{L}\left( {{\delta }_{x}} \right)=-\sum\limits_{a\left( l \right)\in \mathcal{A}}{Q\left( s\left( l \right),a\left( l \right),{{x}_{a}}\left( s\left( l \right);{{\delta }_{x}} \right);{{\delta }_{q}}\left( l \right) \right)},
\end{equation}
\begin{equation}\label{lossq}
    \mathcal{L}\left( {{\delta }_{q}} \right)=\frac{1}{2}{{\left( y\left( l \right)-Q\left( s\left( l \right),a\left( l \right),{{x}_{a}}\left( l \right);{{\delta }_{q}} \right) \right)}^{2}}.
\end{equation}
Based on the (\ref{lossx})(\ref{lossq}) above, the gradient of Q-values estimated by the Q network is backpropagated to the actor-network, which guides it to learn the optimal action parameter policy. To address the issue of joint action parameter input present in the P-DQN framework, MP-DQN introduces a standard basis vector $x_e^a$, where only the $a$-th action parameter $x_a$ is non-zero, while all other action parameters remain zero. As a result, the Q network computes the Q-value solely based on the current action parameter, effectively eliminating the influence of invalid gradients. MP-DQN uses a multi-channel approach that allows small batches of data containing B tuples of $\left|a\right|$ actions to be processed in the same way as small batches of size $\left|a\right|$:
\begin{equation}
\left( 
\begin{aligned}
  & Q\left( s,a,x{{e}_{1}} \right) \\ 
  & Q\left( s,a,x{{e}_{2}} \right) \\ 
  & Q\left( s,a,x{{e}_{\left| a \right|}} \right) 
\end{aligned} 
\right) = 
\begin{pmatrix}
  Q_{11}  & \cdots & 0 \\
  \vdots & \ddots  & \vdots \\
  0 & \cdots & Q_{|a|\,|a|}
\end{pmatrix},
\end{equation}
where, $Q_{i,j}$ represents the Q-value for action $j$ computed during the $i$-th path.  It's important to note that only the diagonal elements $Q_{i,i}$ are deemed valid and are utilized in the final output, denoted as ${{Q}_{a}}\leftarrow {{Q}_{a,a}}$. This selective process ensures that the model focuses on the most relevant Q-values. 

In practice, we employ stochastic gradient descent to minimize the loss functions (\ref{lossx}) and (\ref{lossq}) while training the network. Furthermore, we utilize soft replacement to update the parameters of the target network, which effectively guides the optimization process and ensures the stability of network updates.

\begin{figure}[t]
\centering
\includegraphics[width=1\columnwidth]{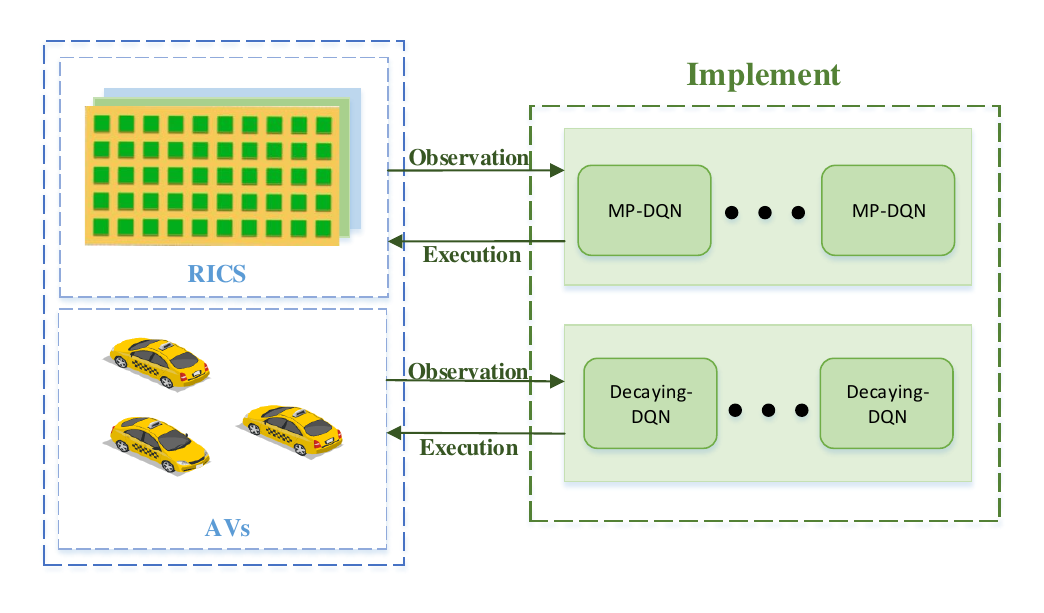}
\caption{The implementation phase of the DS-MADRL algorithm.}
\label{implement}
\end{figure}

Specifically, we employ the CTDE framework, which consists of two classes of agents engaged in collaborative learning. In the \textbf{training phase}, illustrated in Fig.~\ref{training_frame}, agents gather and share observed state information, allowing them to optimize their strategies and learning processes from a global perspective. This centralized training approach ensures that all agents contribute to the optimization of the global objective function using collective information. During this phase, the agents learn cooperative strategies that maximize the global reward $G_t$. 
In the \textbf{execution phase}, depicted in Fig.~\ref{implement}, the trained network is deployed on the individual agents, which make decisions independently based on their observed states and learned policies without relying on information from other agents. This significantly reduces the communication and computational burden during online execution. The specific training process is summarized in \textbf{Algorithm \ref{a1}}.

\section{Numerical Results and Discussion}
In this section, we validate the effectiveness of the proposed DS-MADRL for the autonomous driving network optimization problem. We assume each agent has perfect Channel State Information (CSI), which is updated in every time slot.

\subsection{Simulation Setup}
In the considered simulation scenario, as shown in Fig. \ref{simulation_scene}, one BS is located at $(0,0)$ as the central. RICSs are uniformly distributed in a circular region with a radius of $80$ meters. The initial positions of the AVs and V2V pairs are randomly distributed within a rectangular area of $100$ meters in length and $40$ meters in width, at distances ranging from $250$ to $350$ meters from the origin. The vehicles move at a speed of $10 m/s$. Specifically, each cell is equipped with one RICS, $U$ AVs, and $V$ V2V pairs, with parameters related to the channel and noise detailed in Table \ref{parameters1}. The network-related parameters for our proposed algorithm are presented in Table \ref{parameters2}. 
\begin{table}[thb]
	\centering
	\caption{System model Parameters}
	\label{parameters1}
	\small
	\begin{tabular}{l | l || l | l}
		\noalign{\global\arrayrulewidth=0.3mm}
		\hline
		\textbf{Parameter}  &\textbf{Value } &\textbf{Parameter}  & \textbf{Value } \\
		\hline
		 $K$   & 30 &  $U$   &10\\
		 $V$  &2 &   $Q$&  $2$\\
         $h$  &2  &    $\varsigma_{u,c}$&   [5, 8] GHz\\
		$f_{u,c}$  & [1, 5] GHz & $Penalty$ & 10 \\
        $s_{u,c}$     &[1, 3] Mbits & $P_{outage}$        &0.01\\   
		 $P_u$   &29 dBm  & $P_v$    &     22 dBm\\
		${\gamma}_{th}$ &2 bps/Hz& $\alpha$ &2.5\\       
		$F_{u,c}$   &50 GHz    &  $W{{\xi }_{0}}$ &    -110 dBm\\
            $\lambda$ &0.7&   $A_B(Q)$ &0.8 \\
		\noalign{\global\arrayrulewidth=0.3mm}
		\hline
	\end{tabular}
\end{table}	
\begin{table}[h]
	\centering
	\caption{Network Parameters}
	\label{parameters2}
	\small
	\begin{tabular}{l | l || l | l}
		\noalign{\global\arrayrulewidth=0.3mm}
		\hline
		\textbf{Parameter}  &\textbf{Value } &\textbf{Parameter}  & \textbf{Value } \\
		\hline
        $Episodes$   &600 &   $Step$   &200\\
        $Batch\ size$ &32&  $Learning\ rate$ &0.0001\\
        $decay\ rate$ &0.999& $memory\ size$ &5000\\
        $\varepsilon$ &0.999& $\gamma$ &0.95\\
		\noalign{\global\arrayrulewidth=0.3mm}
		\hline
	\end{tabular}
\end{table}

\subsection{Convergence Performance}
This section primarily focuses on the impact of hyperparameters on the DRL algorithm. We employ two comparative algorithms to validate the convergence performance of our proposed algorithm:
\begin{itemize}
    \item \textbf{DDPG+DQN}: We utilize the DDPG algorithm to explore the optimal policy for the continuous variable offloading ratio, which is managed using the MP-DQN in our algorithm.
    \item \textbf{MP-DQN+DQN}: The choose of $\mathbf{\Theta}$ for RICS is performed using the DQN algorithm.
\end{itemize}

\begin{figure}[t]
	\centering
	\includegraphics[width=1.05\columnwidth]{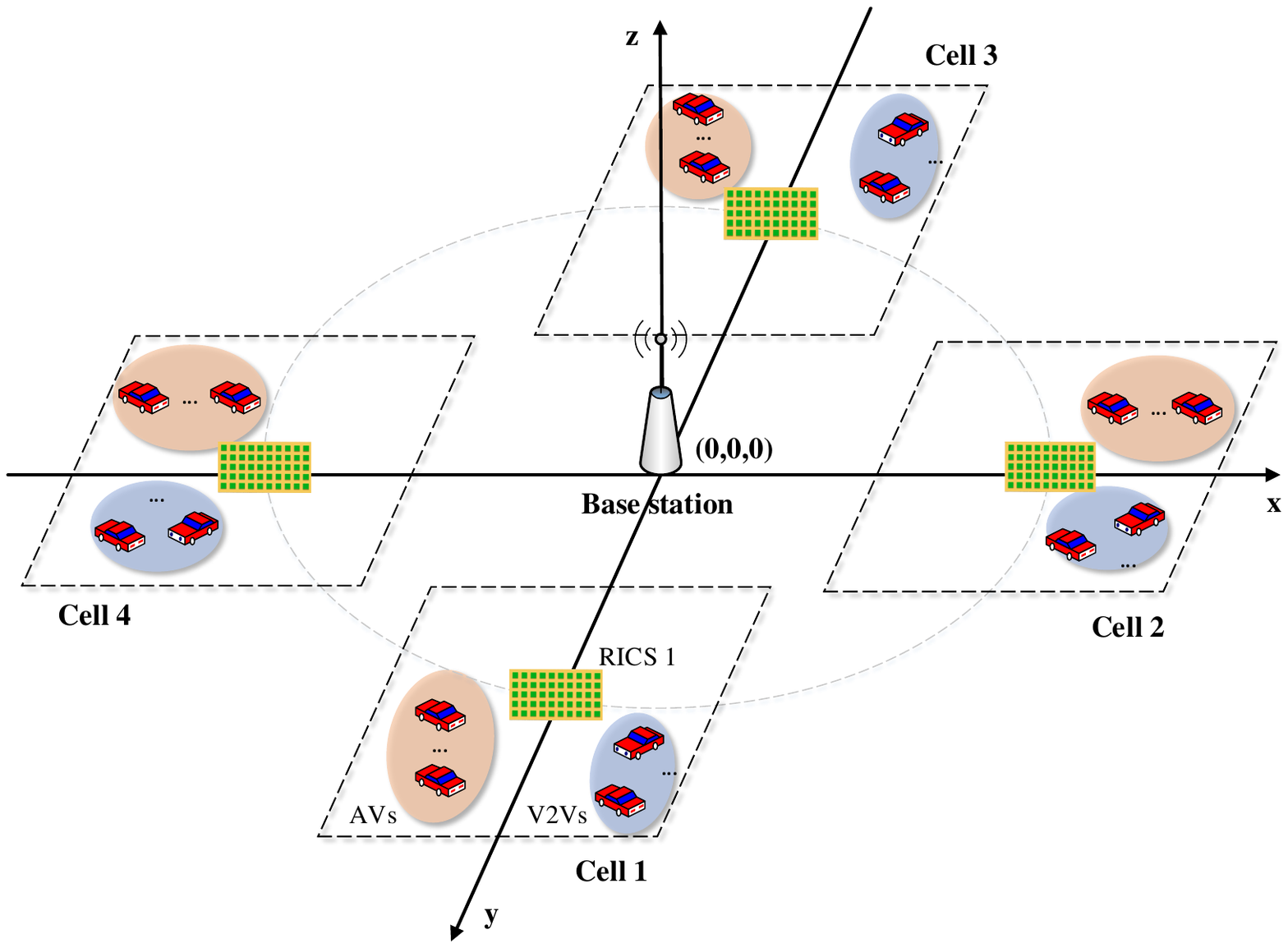}
	\caption{Simulation setup for RICS-Assisted autonomous vehicular network.}
    \label{simulation_scene}
\end{figure}

\begin{figure*}[t] 
    \centering
    \subfloat[]
    {
    \includegraphics[width=0.6\columnwidth]{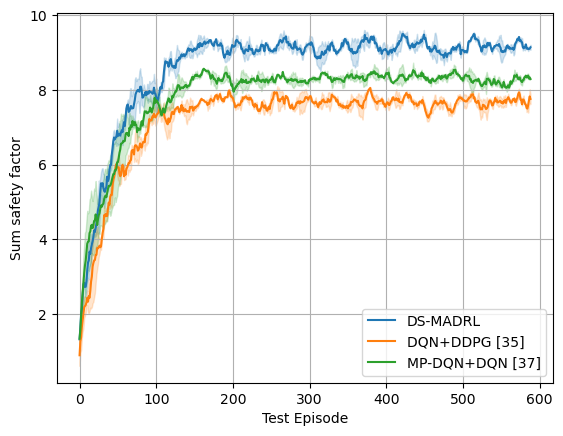}
     \label{train}
    }
    \hfill
    \subfloat[]
    {
    \includegraphics[width=0.60\columnwidth]{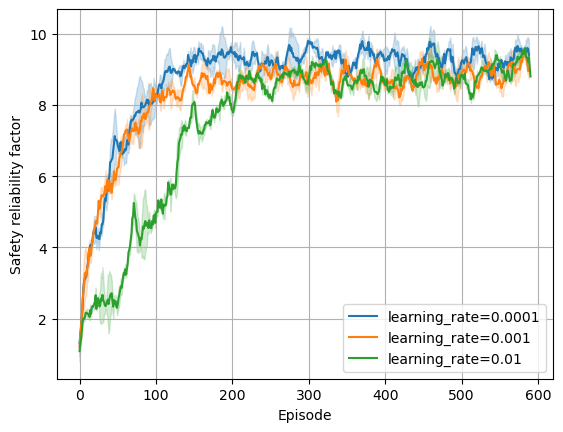}
    \label{learning_rate}}
    \hfill
    \subfloat[]
    {
    \includegraphics[width=0.60\columnwidth]{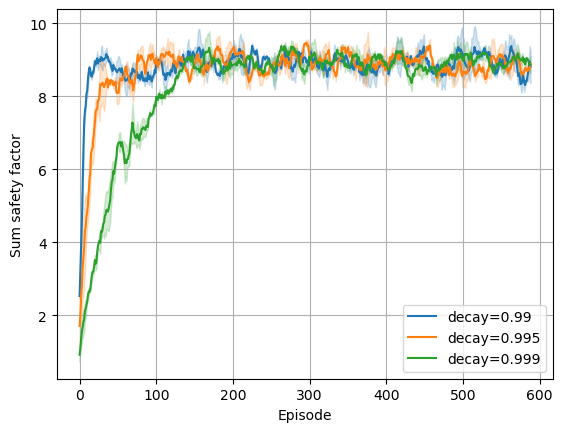}
    \label{lr_decay}}
    \caption{The convergence of different algorithms is shown in (a), the influence of different learning rates for convergence is shown in (b), and the effect of different decay rates is shown in (c).}
\end{figure*}

\begin{figure}[t]
	\centering
	\includegraphics[width=0.98\columnwidth]{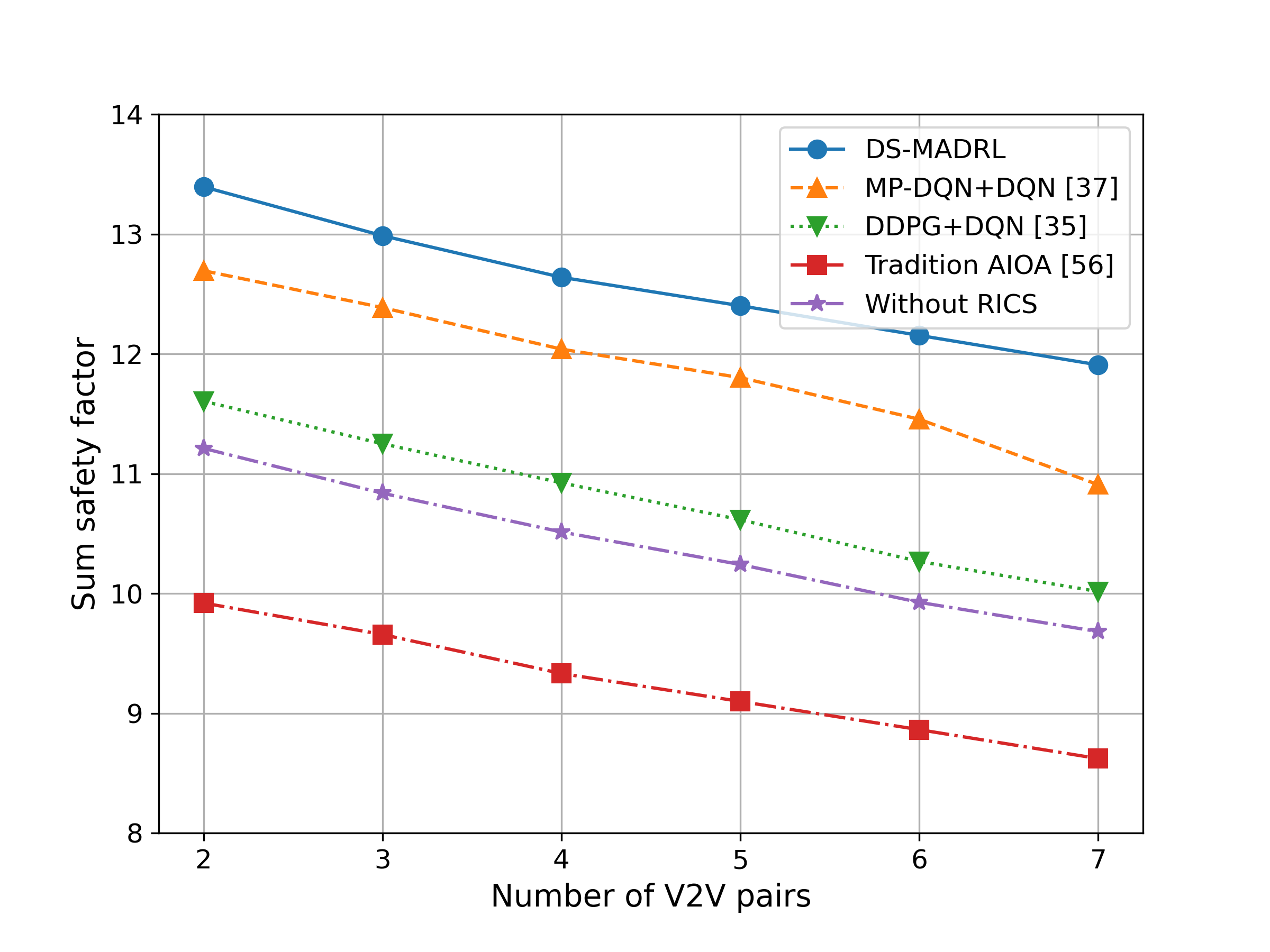}
	\caption{Sum safety factor of AVs versus varying V2V pairs.}
    \label{V2V}
\end{figure}

\begin{figure}[t]
	\centering
	\includegraphics[width=0.98\columnwidth]{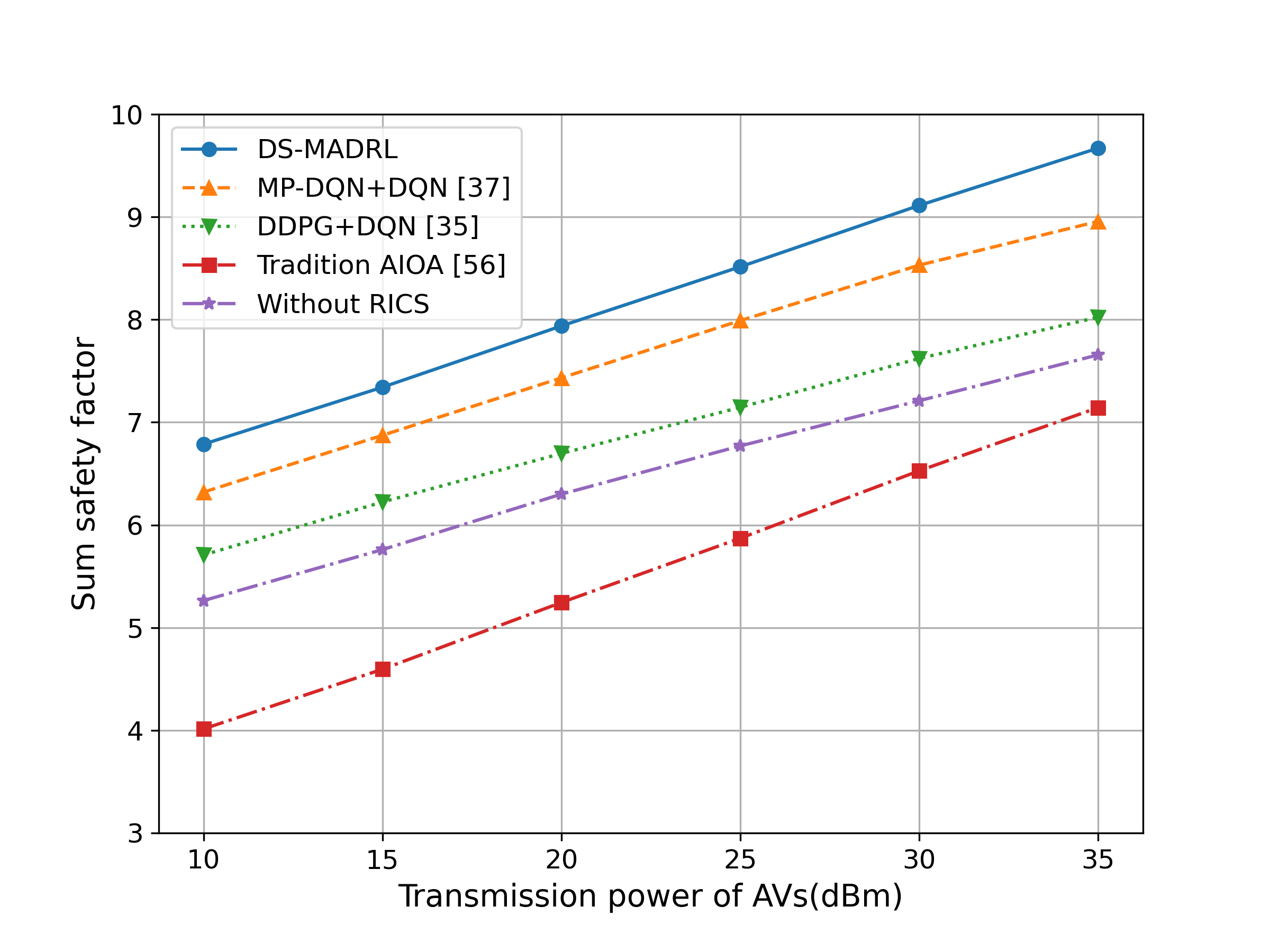}
	\caption{Sum safety factor of AVs versus varying transmission power $P_u$.}
    \label{Pu}
\end{figure}

\begin{figure}[t]
	\centering
	\includegraphics[width=0.9\columnwidth]{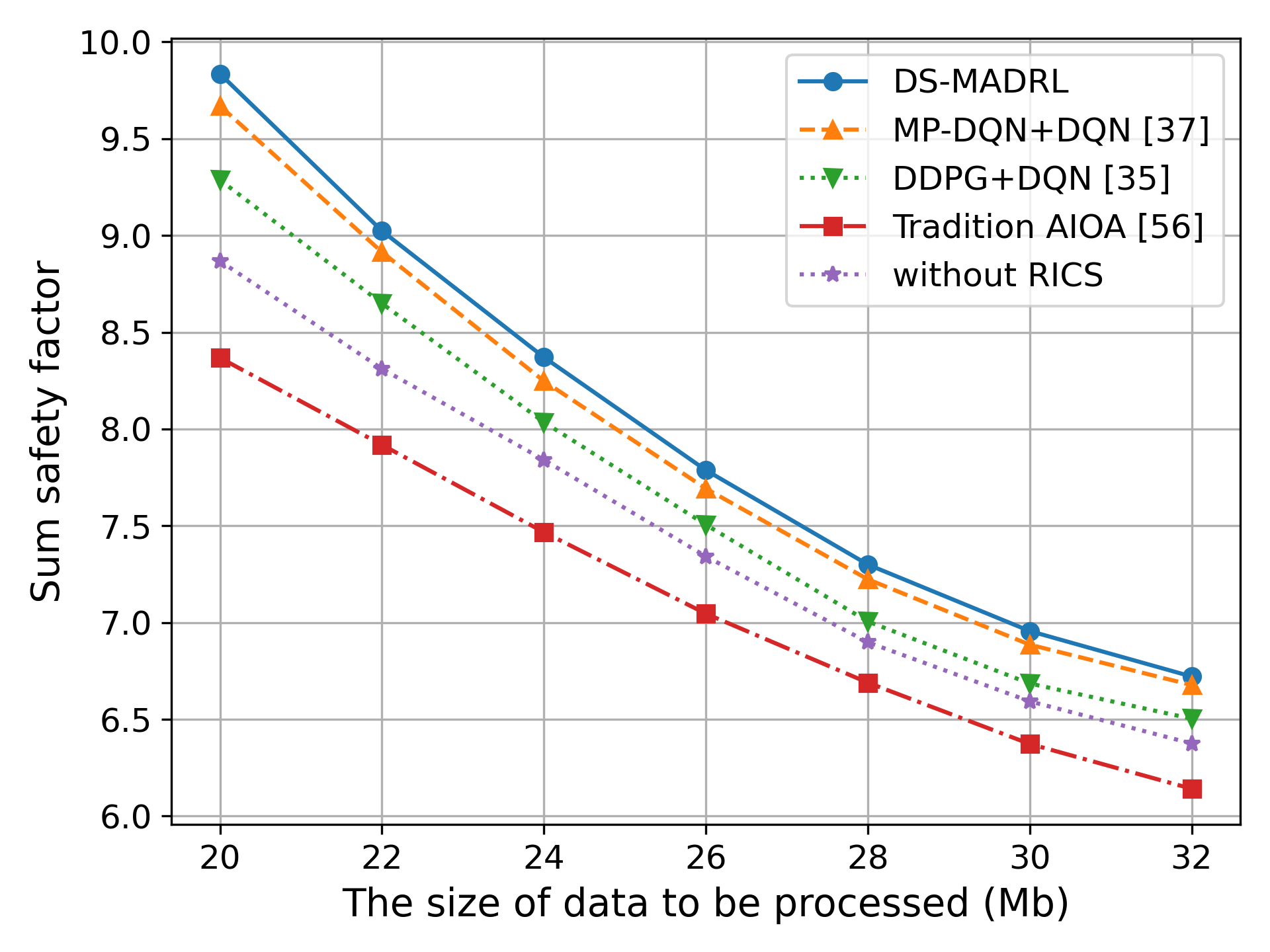}
	\caption{Sum safety factor of AVs versus varying computation data size $s_u$.}
    \label{diff_sm}
\end{figure}

In the experiments, we employed a three-layer Deep Neural Network (DNN), where the two intermediate layers consist of $64$ and $32$ hidden neurons, respectively. The input layer is designed to receive state features, while the output layer generates Q-values for the action space. Additionally, we utilized the Rectified Linear Unit (ReLU) activation function to introduce non-linearity into the model, and the Adam optimizer was employed to minimize the loss function effectively.

Under the condition of one cell, Fig.~\ref{train} illustrates the convergence of the DS-MADRL algorithm compared to two benchmark algorithms when $U=10$ and $V=2$. It is evident that all three algorithms achieve convergence within $200$ iterations, with a final total safety factor of $[9.03, 8.21, 7.59]$ and an average safety factor of $[90.3\%, 82.1\%, 75.9\%]$. Clearly, our algorithm demonstrates superior performance, achieving an improvement of $[8.2\%,14.4\%]$ over the two benchmark algorithms. This advantage may be attributed to the lack of a decay strategy in $\textbf{MP-DQN+DQN}$, which results in lower exploration efficiency in complex scenarios compared to the DDQN. Furthermore, the combined optimization complexity of the $\textbf{DDPG+DQN}$ algorithm is relatively high, as this fully continuous optimization method relies on gradient updates of the policy and can struggle with training difficulties and performance degradation due to the large action space. In contrast, the MP-DQN conducts a forward pass for each discrete action, allowing for more precise optimization of action parameters within high-dimensional action spaces.

Fig.~\ref{learning_rate} compares the impact of varying learning rates on algorithm convergence under identical experimental conditions. It is observed that a lower learning rate of $0.0001$ facilitates smoother model updates. A learning rate of $0.001$ also achieves stable convergence, although its convergence performance is inferior to that of $0.0001$. In contrast, a learning rate of $0.01$ results in slower and less stable convergence.

Fig.~\ref{lr_decay} explores the impact of different decay rates on algorithm convergence. This parameter governs the trade-off between ``exploration" and ``exploitation" when the agent selects actions. ``Exploration" indicates the agent randomly chooses actions to discover potentially better strategies, while "exploitation" signifies that the agent chooses the currently known optimal action. A decay rate that is too low may result in insufficient exploration of new possibilities, making it susceptible to local optima. Additionally, the safety factors correspond to the three decay rates, which are about $[89.5\%, 89.7\%, 90.4\%]$. A lower decay rate allows for a quicker shift to the exploitation phase during decision-making, relying on the current best strategy. However, the convergence results are not as favorable as those achieved with a larger decay rate. Conversely, a larger decay rate means the agent must spend more time exploring during the initial phase, leading to a slower convergence speed.

\subsection{Results for Single-Cell Systems}
\begin{figure*}[h]
	\centering
    \subfloat[]
    {
    \includegraphics[width=0.99\columnwidth]{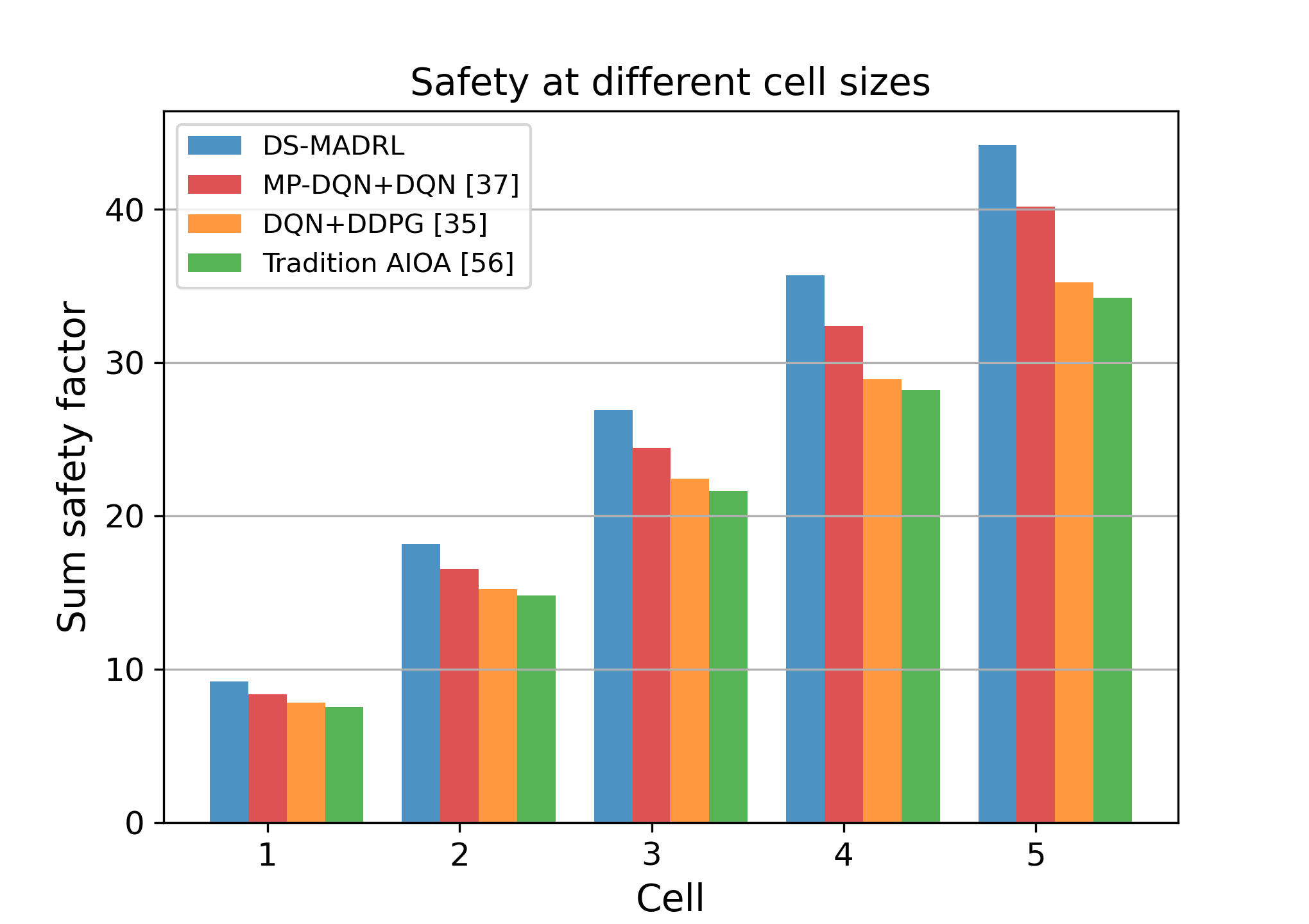}
    }
    \hfill
    \subfloat[]
    {
    \includegraphics[width=0.99\columnwidth]{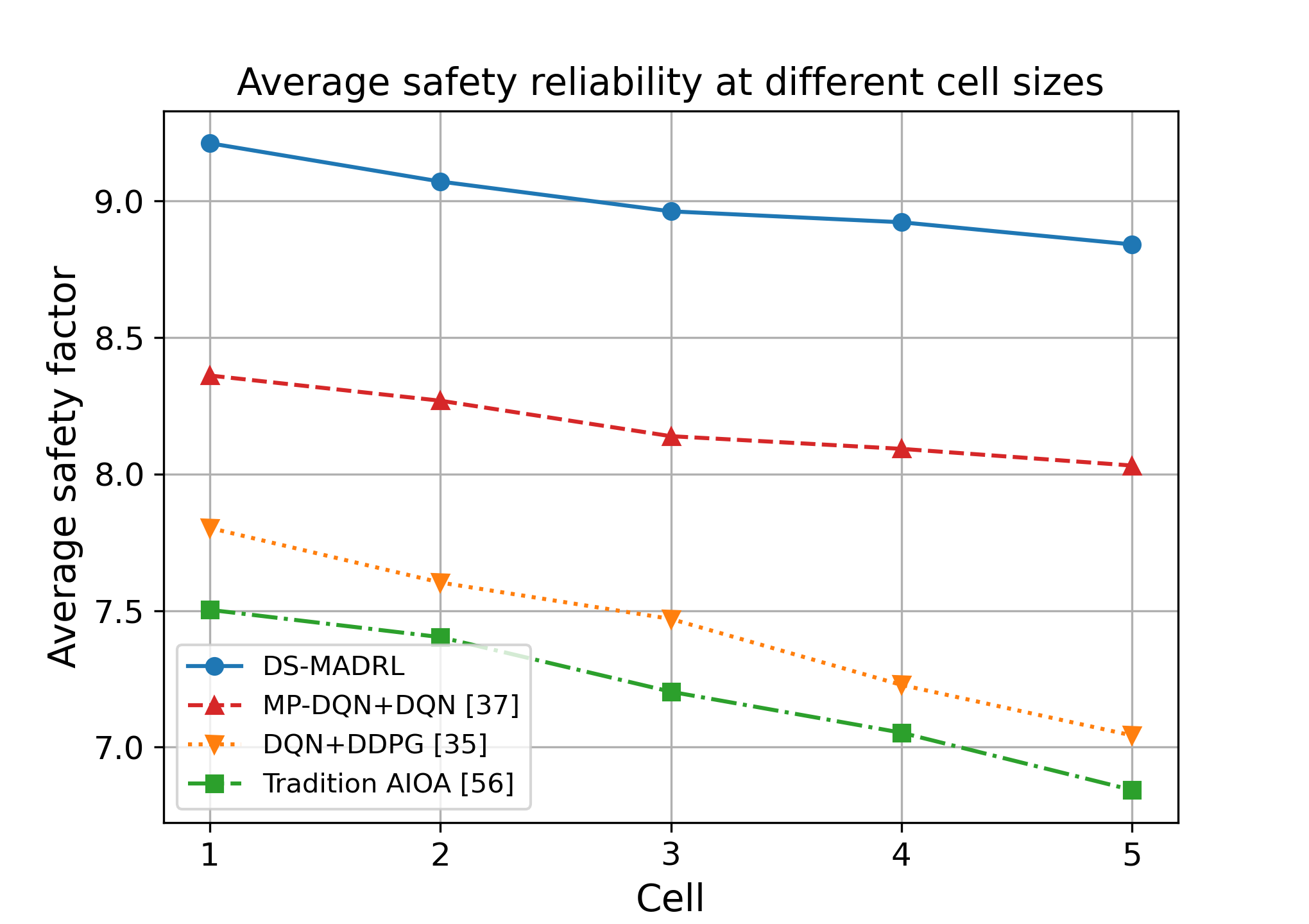}
    }
	\caption{Sum safety factor versus the number of cells is shown in (a), the average safety factor is shown in (b) in our system.}
    \label{block}
\end{figure*}

The experimental results indicate that the choice of hyperparameters significantly impacts the convergence speed and final performance of the proposed algorithm. Therefore, selecting an appropriate combination of hyperparameters is crucial to ensure that the algorithm converges quickly and reaches a high-performance ceiling, thereby effectively enhancing safety.

In this subsection, we explore the impact of various factors—including the transmission power of AVs, the number of V2V pairs, and the size of data to be processed by AVs in one cell on the system performance are explored separately. 
\begin{itemize}
    \item \textbf{Traditional AIOA}: The traditional convex optimization algorithm alternating optimization method is used, whose core idea is to decompose the original problem into three subproblems and optimize only one of them in each iteration until the whole algorithm converges.
    \item \textbf{Without RICS}: Our network scenario does not include assistance from RICS. The algorithm used is still DS-MADRL.
\end{itemize}

Fig.~\ref{V2V} illustrates a slight decrease in the safety factor of AVs (U=15) as the number of V2Vs increases. This phenomenon can be attributed to the multiplexing of spectral resources between the V2V communication link and the V2I link, which introduces significant interference, leading to a deterioration in link quality and a subsequent reduction in throughput. Similar to the conclusion in the previous subsection, the proposed algorithm continues to outperform the other three DRL algorithms and demonstrates a substantial advantage over conventional optimization algorithms, achieving a $17\%$ improvement in the average safety factor of the AVs.

Meanwhile, as depicted in Fig.~\ref{Pu}, an increase in the transmission power of AVs directly enhances the signal strength of the V2I links, thereby mitigating the effects of path loss on link quality. A higher V2I signal power also improves the SINR, resulting in enhanced reliability of the link, and the data rate and computational efficiency. Additionally, the conventional optimization algorithms perform significantly worse than the DRL algorithm.

Fig.~\ref{diff_sm} examines the safety performance of AVs in different data sizes and algorithms. As the $s_u$ increases, there is an overall decreasing trend in the safety factor, with a more rapid decline observed when the data size is smaller. The rate of decrease tends to stabilize with a gradual convergence trend as the data size continues to increase. Furthermore, the performance of the DS-MADRL algorithm consistently outperforms that of the comparison algorithms. Initially, as the data size increases, the demand for data transmission occupies the bandwidth of the V2I link, leading to the introduction of V2V interference. This results in a sharp decline in channel quality and a corresponding decrease in the safety factor. As the task volume increases further, the critical resources in the system approach the upper limit and the channel quality and interference reach a relatively balanced state.

\begin{figure}[t]
	\centering
	\includegraphics[width=0.98\columnwidth]{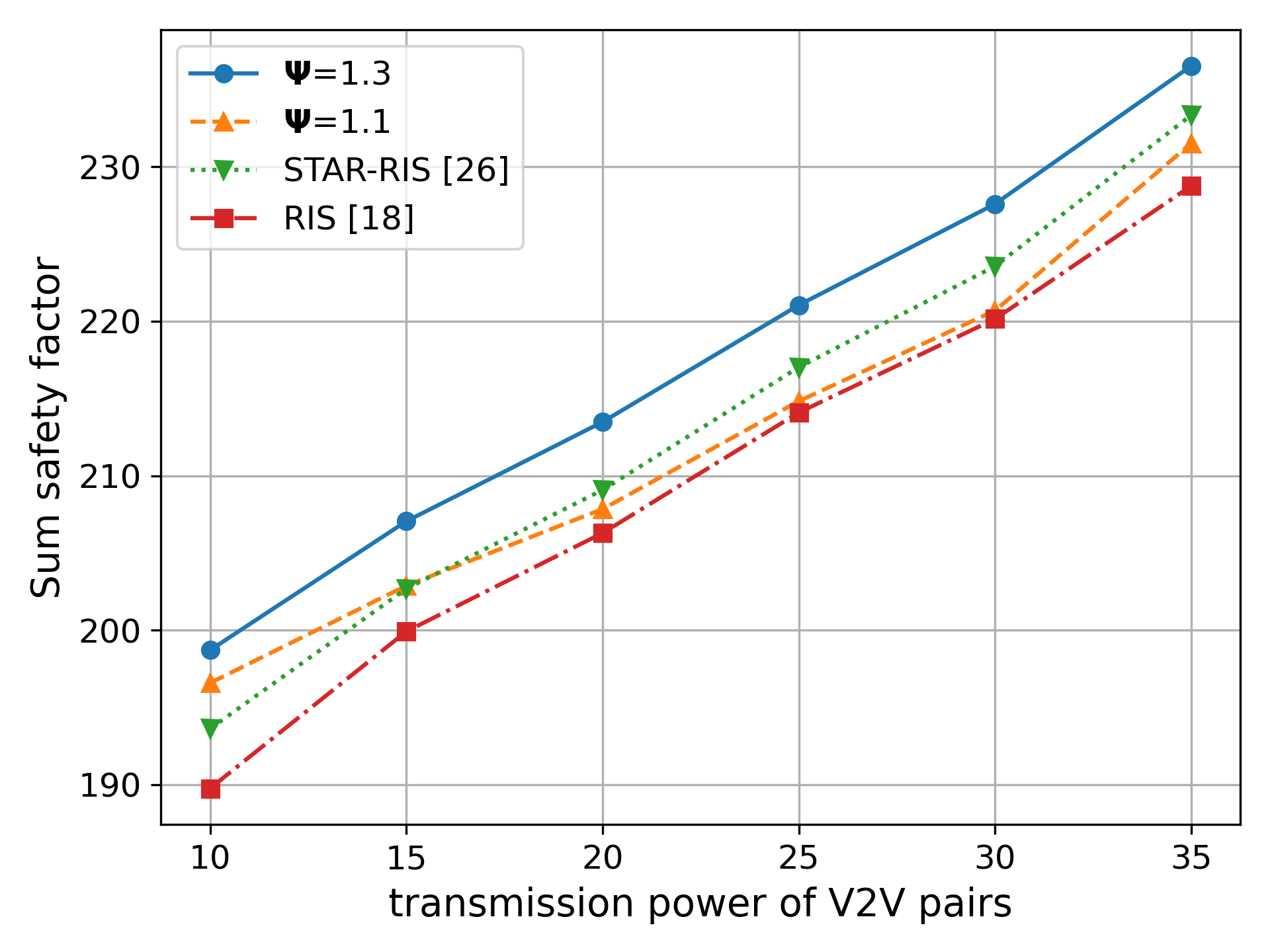}
	\caption{Performance of different RICS $\bm{\Psi}$ under varying $P_v$.}
    \label{Psi}
\end{figure}

\subsection{Results for Multi-Cell Systems}
Fig.~\ref{block} characterizes the total safety factor of AVs across varying cell configurations. The left panel illustrates the cumulative safety factor for all cells, while the right panel depicts the average safety factor per cell. The experimental results indicate that as the number of cells increases, the total safety factor of different algorithms rises significantly. However, the average safety factor per cell experiences a slight decline due to the increased load on the BS and the longer waiting times associated with servicing a greater number of cells. Notably, this decrease is constrained to a range of $1\% \sim 2\%$, reflecting the robustness and stability of the algorithm even under high-load conditions. The findings demonstrate that the proposed DS-MADRL algorithm can effectively address the computational resource allocation challenges in multi-cell scenarios, ensuring the security and reliability of the system are maintained.

\subsection{Impact of amplitude adjustment factor $\bm{\Psi}$ for RICS}

In this part, we compare the effects of different types of RIS in V2V data rates and verify the role of different amplitude adjustment factors $\bm{\Psi}$ on interference mitigation. We compare different $\bm{\Psi}$, as well as STAR-RIS \cite{STAR} and RIS. The specific benchmarks are as follows:
\begin{itemize}
    \item \textbf{RICS with different values}: $\bm{\Psi}=1.1$, $\bm{\Psi}=1.3$, and $\bm{\Psi}=0.8$. (all elements of RICS are equipped with the same $\bm{\Psi}$).
    \item \textbf{STAR-RIS}: The rest of the configuration is the same as RICS, without the signal amplitude adjustment function.
    \item \textbf{RIS}: Only possesses signal reflection capabilities.
\end{itemize}

Fig.~\ref{Psi} illustrates the impact of different amplitude adjustment factors on the V2V data rate. As expected, it is observed that the data rate of V2V pairs increases with the rising of $P_v$. Additionally, different amplitude adjustment factors $\bm{\Psi}$ have varying effects on V2V data rates; appropriate configurations can effectively mitigate interference from the V2I link. For instance, $\bm{\Psi}=1.3$ is particularly effective in mitigating such interference. While certain parameter configurations may negatively impact interference mitigation. Notably, when $\bm{\Psi}=1.1$, the V2V pair data rate is inferior to that of the other configurations. 

Moreover, well-chosen $\bm{\Psi}$ values outperform STAR-RIS and traditional RIS, which lack interference cancellation capabilities. Among the tested configurations, $\bm{\Psi}=1.3$ achieves the best performance, yielding improvements in data transmission rates of $[1.97\%, 2.36\%, 3.43\%]$ compared to the other four schemes. This demonstrates that RICS can leverage its signal adjustment capabilities to effectively mitigate interference experienced by V2V communications, thereby enhancing data rates and system security.

\section{Conclusions}
This paper studied RICS-assisted autonomous driving under safety requirements and presented a novel DS-MADRL scheme. The considered design optimization problem was modeled as an MG process, utilizing MP-DQN to handle the continuous-discrete hybrid action space of AVs and employing DDQN for the discrete phase configuration selection of RICS. The proposed approach enabled effective joint decision-making through collaborative interactions with the environment. The convergence performance of the presented joint learning framework was investigated via extensive simulation experiments, which also unveiled the impact of various system parameters on the overall performance. It was demonstrated that the proposed approach maintains robustness and adaptability across different cell scenarios, enhancing significantly the overall performance of the system. 


\ifCLASSOPTIONcaptionsoff
  \newpage
\fi

\vspace{12pt}


%

%
%
%




\end{spacing}
\end{document}